\documentclass{article}

\usepackage{microtype}
\usepackage{graphicx}
\usepackage{subfigure}
\usepackage{booktabs} 
\usepackage{amssymb,amsmath, amsthm}
\usepackage{epstopdf}
\usepackage{hyperref}
\usepackage{algorithm}
\usepackage{algorithmic}

\numberwithin{equation}{section}

\usepackage[dvipsnames]{xcolor}

\newcommand{\ip}[1] {\langle #1 \rangle }
\newcommand{\norm}[1] {\left \| #1 \right \|}
\newcommand{\inclu}[0] {\ar@{^{(}->}}

\newcommand{\RR}{\mathbb{R}}

\newcommand{\R}{{\bf R}}
\newcommand{\EE}{{\mathbb E}}

\newcommand{\cV}{\mathcal{V}}

\newcommand{\cL}{\mathcal{L}}

\newcommand{\cO}{\mathcal{O}}

\newcommand{\cS}{\mathcal{S}}
\newcommand{\cA}{\mathcal{A}}
\newcommand{\cP}{\mathcal{P}}

\newcommand{\argmin}{\operatornamewithlimits{argmin}}

\providecommand{\Ex}[1]{\mathbb{E}\left[#1\right]}

\newtheorem{theorem}{Theorem}[section]
\newtheorem{proposition}[theorem]{Proposition}
\newtheorem{lemma}[theorem]{Lemma}

\newtheorem{corollary}[theorem]{Corollary}
\newtheorem{assumption}[theorem]{Assumption}

\newtheorem{example}{Example}[section]

\usepackage{mathtools}

\usepackage{jmlr2e} 

\begin{document}

\jmlrheading{1}{2020}{pp}{mm/dd}{mm/dd}{Junyu Zhang, Amrit Singh Bedi, Alec Koppel, and Mengdi Wang}


\ShortHeadings{Distributional Risk in the Dual Domain}{Zhang, Bedi, Koppel, and Wang}
\firstpageno{1}

\title{Cautious Reinforcement Learning \\ via Distributional Risk in the Dual Domain}

\author{\name Junyu Zhang\thanks{Denotes equal contribution.}  \email zhan4393@umn.edu\\
\addr Department of Industrial and Systems Engineering\\
 University of Minnesota\\
  Minneapolis, Minnesota, 55455
       \AND
       \name Amrit Singh Bedi\footnotemark[1] \email amrit0714@gmail.com\\
       \addr Computational and Information Sciences Directorate\\
        US Army Research Laboratory \\
         Adelphi, MD, USA 20783
              \AND
        \name Mengdi Wang \email mengdiw@princeton.edu  \\
       \addr Department of Electrical Engineering \\
       Center for Statistics and Machine Learning\\
       Princeton University\\
       Princeton, NJ 08544
       \AND
       \name Alec Koppel \email alec.e.koppel.civ@mail.mil\\
       \addr Computational and Information Sciences Directorate\\
        US Army Research Laboratory \\
         Adelphi, MD 20783 }


\editor{}

\maketitle

\begin{abstract} We study the estimation of risk-sensitive policies in reinforcement learning problems defined by a Markov Decision Process (MDPs) whose state and action spaces are countably finite. Prior efforts are predominately afflicted by computational challenges associated with the fact that risk-sensitive MDPs are time-inconsistent. To ameliorate this issue, we propose a new definition of risk, which we call \emph{caution}, as a penalty function added to the dual objective of the linear programming (LP) formulation of reinforcement learning.  {The caution measures the distributional risk of a policy, which is a function of the policy's long-term state occupancy distribution.} To solve this problem in an online model-free manner, we propose a stochastic variant of primal-dual method that uses Kullback-Lieber (KL) divergence as its proximal term. We establish that the number of iterations/samples required to attain approximately optimal solutions of this scheme matches tight dependencies on the cardinality of the state and action spaces, but differs in its dependence on the infinity norm of the gradient of the risk measure. Experiments demonstrate the merits of this approach for improving the reliability of reward accumulation without additional computational burdens.

\end{abstract}

\section{Introduction}\label{Intro}
In reinforcement learning (RL) \citep{sutton2018reinforcement}, an autonomous agent in a given state selects an action and then transitions to a new state randomly depending on its current state and action, at which point the environment reveals a reward. This framework for sequential decision making has gained traction in recent years due to its ability to effectively describe problems where the long-term merit of decisions does not have an analytical form and is instead observed only in increments, as in recommender systems \citep{karatzoglou2013learning}, videogames \citep{mnih2013playing,vinyals2019alphastar}, control amidst complicated physics \citep{schulman2015high}, and management applications \citep{peidro2009quantitative}. 

The canonical performance metric for RL is the expected value of long-term accumulation of rewards. Unfortunately, restricting focus to expected returns fails to encapsulate many well-documented aspects of reasoning under uncertainty such as anticipation \citep{roca2011identifying}, inattention \citep{sims2003implications}, and risk-aversion \citep{tom2007neural}. In this work, we focus on risk beyond expected rewards, both due to its inherent value in behavioral science and in pursuit of improving the reliability of RL in safety-critical applications \citep{achiam2017constrained}.

Risk-awareness broadens the focus of decision making from expected outcomes to other quantifiers of uncertainty. Risk, originally quantified using the variance in portfolio management  \citep{markowitz1952portfolio}, has broaden to higher-order moments or quantiles \citep{rockafellar2002conditional}, and gave rise to a rich theory of coherent risk \citep{artzner1999coherent}, which has gained attention in RL in recent years \citep{chow2017risk,jiang2018risk} as a frequentist way to define uncertainty-aware decision-making. 

Incorporating risk gives rise to computational challenges in RL. In particular, if one replaces the expectation in the value function by a risk measure, the MDP becomes \emph{time-inconsistent} \citep{bjork2010general}, that is, Bellman's principle of optimality does not hold. This issue has necessitated modified Bellman equations \citep{ruszczynski2010risk}, multi-stage schemes \citep{jiang2018risk}, or policy search \citep{tamar2015policy}, all of which do not attain near-optimal solutions in polynomial time, even for finite MDPs.
Alternatively, one may impose risk as a probabilistic constraint \citep{krishnamurthy2003implementation,prashanth2014policy,chow2017risk,paternain2019learning,yu2019convergent}
 in the spirit of chance-constrained programming \citep{nemirovski2007convex} 
 common in model predictive control.  

 An additional approach is Bayesian \citep{ghavamzadeh2015bayesian} and distributional RL \citep{bellemare2017distributional}, which seeks to track a full posterior over returns. These approaches benefit from the fact that with access to a full distribution, one may define risk specifically, with, e.g., conditional value at risk (CVaR)\citep{keramati2019}.
  One limitation is that succinctly parameterizing the value distribution intersects with approximate Bayesian computation, an active area of research \citep{yang2019fully}. 

 In this paper, we seek to define risk in sequential decision making that (1) provides a tunable tradeoff between the mean return and uncertainty of a decision; (2) captures long-term behaviors of policies that cannot be modeled using cumulative functions; (3) can be solved efficiently in polynomial time, depending on the choice of risk.
 To do so, we formulate a class of distributional risk-averse policy optimization problems to address risks involving the long-term behaviors that permit the derivation of efficient algorithms. More specifically, we:
  \begin{itemize}
  \item  propose a new definition of the risk of a policy, which we call \emph{caution}, as a function of the policy's long-term state-action occupancy distribution. We formulate a caution-sensitive policy optimization problem by adding the caution risk as a penalty function  to the dual objective of the linear programming (LP) formulation of RL. The caution-sensitive optimization problem is often convex, allowing us to directly design the policy's long-term occupancy distribution (Sec. \ref{sec:risk}).
  \item derive an online model-free algorithm based on a stochastic variant of primal-dual policy gradient method that uses Kullback-Lieber (KL) divergence as its proximal term, and extend the method to nonconvex caution risks by using a block coordinate ascent (BCA) scheme (Sec. \ref{sec:algorithm}). 
  \item establish that the number of sample transitions required to attain approximately optimal solutions of this scheme matches tight dependencies on the cardinality of the state and action spaces, as compared to the typical risk-neutral setting (Sec. \ref{sec:convergence}). 
  \end{itemize}
Further, we demonstrate the experimental merits of this approach for improving the reliability of reward accumulation without additional computational burdens (Sec. \ref{sec:experiments})

\section{Preliminaries}

\subsection{Discounted Markov Decision Process}
We consider the problem of reinforcement learning (RL) with finitely many states and actions as mathematically described by a Markov Decision Process (MDP) $ (\cS,\cA,\cP, r, \gamma)$. For each state $i\in \cS$, a transition to state $j\in \cS$ occurs when selecting action $a\in\cA$ according to a conditional probability distribution $j\sim \cP(\cdot | a, i )$, for which we define the short-hand notation $P_{a}(i,j)$. Moreover, a reward $\hat r:\cS \times \cS \times \cA \mapsto \RR$ is revealed and is denoted as $\hat r_{ija}$. Without loss of generality, we assume $\hat r_{ija}\in[0,1]$ with probability 1 for  $\forall i,j\in\cS$ and $\forall a\in\cA$ throughout the paper. For future reference, we denote the expected reward with respect to transition dynamics as $r_{ia}:=\EE\left[\hat r_{ija}|i,a\right] = \sum\limits_{j\in\cS}P_{a}(i,j)\cdot \hat r_{ija}$ and the vector of rewards for each action $a$ as $r_a = [r_{1a},\cdots,r_{|\cS|a}]^T\in\RR^{|\cS|}$.

 In standard (risk-neutral) RL, the goal is to find the action sequence which yields the most long-term reward, or value:
\begin{equation}
\label{eq:value}
v^*(s):=\max_{\{a_t \in\mathcal{A}\} }\mathbb{E}\left[\sum_{t=0}^\infty \gamma^t \hat r_{i_t i_{t+1} a_t}~\bigg|~ i_0 = s\right], \quad\forall s\in\cS.
\end{equation}

\subsection{Bellman Equation and Duality}
The optimal value function $v^*$  \eqref{eq:value} satisfies Bellman's optimality principle \cite{bertsekas2004stochastic}:
\begin{align}\label{main}
\!\!v^*(i)\!=\!\max_{a\in\cA}\Big\{\gamma\!\sum\limits_{j\in\cS}\!P_{a}(i,j)v^*(i)\!+\!\sum\limits_{j\in\cS}P_{a}(i,j)\hat r_{ija}\!\Big\} 
\end{align}
%
 for all $i \in\cS$. Then, due to \cite{de2003linear}, the Bellman optimality equation \eqref{main} may be reformulated as a linear program (LP)
\begin{eqnarray}\label{primal}
&\min_{v\geq 0} &  \langle \xi,v\rangle\\
&\text{s.t.} 
&(I-\gamma P_a)v-r_a \geq 0, \ \ \forall a\in\cA\nonumber
\end{eqnarray}
where $\xi$ is an arbitrary positive vector. The dual of \eqref{primal} is given as
\begin{eqnarray}\label{dual}
&\max_{\lambda\geq 0} & \sum\limits_{a\in\cA}\langle\lambda_a,r_a\rangle 
\\
&\text{s.t.} &\sum\limits_{a\in\cA}(I-\gamma P_a^\top)\lambda_a=\xi, \ \ \,\,\forall a\in\cA\nonumber
\end{eqnarray}
where $\lambda_a = [\lambda_{1a},\cdots,\lambda_{|\cS|a}]^\top\in\RR^{|\cA|}$ is the $a$-th column of $\lambda$. Essential to the subsequent development is the fact that $\lambda$ is an unnormalized state-action occupancy measure and 
$$\sum_{a\in\cA}\langle\lambda_a,r_a\rangle = \mathbb{E}\left[\sum_{t=0}^\infty \gamma^t r_{i_t i_{t+1} a_t}~\bigg|~ i_0 \sim \xi, a_t\sim \pi(\cdot|i_t)\right]$$ when $\xi$ belongs to the probability simplex. Moreover, one can recover the policy parameters through normalization of the dual variable as 
$\pi(a|s) = {\lambda_{sa}}/{\sum_{a'\in\cA}\lambda_{sa'}} $ for all $ a\in\cA$ and $ s\in\cS$, as detailed in Proposition \ref{proposition:dual-meaning}.

\section{Caution-Sensitive Policy Optimization}\label{sec:risk}


In this work, we prioritize definitions of risk in MDPs that capture long-term behavior of the policy and permit the derivation of computationally efficient algorithms. We focus on optimizing the policy's long-run behaviors that cannot be described by any cumulative sum of rewards, for examples the peak risk and variance. 

\subsection{Problem Formulation}
We focus on directly designing the long-term state-action occupancy distribution, whose unnormalized version is the dual variable $\lambda:=\{\lambda_a\}_{a\in\cA}$. Rather than only maximizing the expected cumulative return, i.e., the typical objective in risk-neutral MDP (e.g., \eqref{dual}), we seek policies that incorporate risk functions concerning the full distribution $\lambda$.  

We propose a non-standard notion of risk: in standard definitions, such as those previously mentioned, they are typically risk measures of the cumulative rewards; by contrast, here we augment the risk to be defined over the \emph{long-term state-action occupancy distributions}, which we dub \emph{caution} measures. Specifically, denote as $\rho(\lambda)$ a caution function that takes as input dual variables $\lambda$ (unnormalized state-action distributions) feasible to \eqref{dual} and maps to the reals $\RR$. The caution risk measures the fitness of the entire state path, rather than just a cumulative sum over the path. 

In pursuit of computationally efficient solutions, we hone in on properties of the dual LP formulation of RL. The caution-sensitive variant of  \eqref{dual} then takes the form:
\begin{eqnarray}
\label{prob:dual}
&\max_{\lambda\geq 0} & \ip{\lambda,r}-c\rho(\lambda)\nonumber 
\\
&\text{s.t.} & \sum\limits_{a\in\cA}(I-\gamma P_a^\top) \lambda_a=\xi, \ \ \ \ \ \  \ \ \ \\
& & \|\lambda\|_1 = (1-\gamma)^{-1},\nonumber
\end{eqnarray}
where $c$ is a positive penalty parameter and we take $\xi $ to be the vector of uniform distribution without loss of generality, i.e., $\xi = \frac{1}{|\cS||\cA|}\cdot\mathbf{1}$; and $\|\lambda\|_1:=\sum_{s,a}|\lambda_{sa}|$. The constraints require that $\lambda$ be the unnormalized state-action distribution corresponding to {\it some} policy. The last constraint is implied by $\sum\limits_{a\in\cA}(I-\gamma P_a^\top) \lambda_a=\xi$, but we include it for clarity. When $\rho$ is convex, problem \eqref{prob:dual} is a convex optimization problem that facilitates computationally efficient solutions. 

Denote by $\lambda^*$ the optimal solution to the cautious policy optimization problem \eqref{prob:dual}. This $\lambda^*$ gives the optimal long-term state-action occupancy distribution under the caution risk. 
Let $\pi^*$ be the mixed policy given by 
$$\pi^*(a|s) = \frac{\lambda^*(s,a)}{\sum_{a'} \lambda^*(s,a')} .$$ 
We call this $\pi^*$ the {\it optimal caution-sensitive policy.}  
We remark that with the introduction of the risk measure into the dual form \eqref{prob:dual}, the corresponding primal is no longer the LP problem \eqref{primal} but changes to one that incorporates risk. The optimal caution-sensitive policy $\pi^*$ differs from the optimal policy in the typical risk-neutral setting. Since the LP structure is lost, the optimal risk-sensitive policy $\pi^*$ is not guaranteed to be deterministic.
Moreover, the Lagrangian multipliers, denoted by $v^*$, for the risk-sensitive problem \eqref{prob:dual} is no longer the risk-neutral value vector, meaning that we are solving a \emph{different problem} than \eqref{eq:value}. Indeed, by defining caution in this way, we incorporate long-term distributional risk into the dual domain of Bellman equation, while sidestepping the computational challenges of time-inconsistency.

\subsection{Examples of Caution Risk}
Next, we discuss several examples of the caution risk $\rho$ to clarify the problem setting \eqref{prob:dual}. 

\begin{example}[{\bf Barrier risks}]\label{eg:barrier}\rm
Caution risk can take the form of barriers to guarantee that a policy's long-term behavior meets certain expectations. Two examples follow:

$\bullet$ {\it Staying in safety set.}
Suppose we want to keep the state trajectory within a safety set $\bar S \subset\mathcal{S}$ for more than $1-\delta$ fraction of all time. In light of the typical barrier risk used in constrained optimization, we define 
$$\rho(\lambda) = - \log \left( \lambda(\bar S) - (1-\delta) \right),$$
where $\lambda(\bar S) = (1-\gamma)\sum_{s,a} \lambda(s,a) \mathbf{1}_{s\in \bar S}.$ Since $\lambda(\bar S)$ is linear in $\lambda$, we can verify that the log barrier risk $\rho$ is convex.

$\bullet$ {\it Meeting multiple job requirements.}
Further, suppose there are multiple tasks with strict requirements on their expected returns $\langle \lambda, r_j\rangle \geq b_j$, $j=1,\ldots,m$. One can transform these return constraints into a log barrier given by
$$\rho(\lambda) = - \sum_{j=1}^m \log \left( \langle \lambda, r_j\rangle - b_j  \right).$$
In this way, the optimal caution-sensitive policy will meet all the job requirements for large enough penalty $c$.
\end{example}

\begin{example}[{\bf Peak risk}]\rm\label{eg:peak}
Let $f_1,\ldots,f_m$ be risk functions of $\lambda$. Consider the peak risk defined as
$$\rho(\lambda) = \sup_{j\in[m]} f_j(\lambda).$$ 

$\bullet$ {\it Worst-case exposure to danger areas.}
For example, let $S_1,\ldots S_m$ be known ``danger" sets. If we let $f_j (\lambda) = \lambda(S_j)$ quantify the long-term exposure to $S_j$, the peak risk $\rho$ measures the long-run exposure to the most acute danger. 

$\bullet$ {\it Worst-case multitask performance.}
For another example,  suppose there are $m$ different tasks defined in the same environment with reward functions $r_1,\ldots, r_m$. Let $f_j(\lambda) = -\langle\lambda, r_j\rangle$ be the negative cumulative return for task $j$, and an agent has to do well in all the tasks and be evaluated based on her worst performance. Then the objective is a peak risk:
$$-\rho(\lambda) = \sup_{j\in[m]} -\langle\lambda, r_j\rangle .$$ 
In the preceding examples, $f_j$'s are linear, therefore $\rho$ is always convex. 
\end{example}

\begin{example}[{\bf Variance risk}]\normalfont\label{eg:variance}
In finance applications, one canonical risk concern is the variance of return. To formulate risk as variance, we first note that $\lambda$ is an unnormalized distribution, whose normalized counterpart is denoted as $\hat \lambda := (1-\gamma)\lambda$. Then it holds that $\ip{\hat \lambda, r}$ is the expected reward accumulation. Then, the variance of return per timestep takes the form
	\begin{align}
	\!\!\!\!\!\rho(\lambda)= Var(\hat r_{ss'a}|\lambda)=\EE^{\hat\lambda}\left[  \big(\EE^{\hat\lambda}\left[ \hat r_{ss'a}\right]-\hat r_{ss'a}\big)^2\right]
	\end{align} 
	where $\EE^{\hat\lambda}:=\EE_{(s,a,s')\sim\hat\lambda\times\cP(\cdot|a,s)}$. For ease of notation, denote $R\in\RR^{|\cS|\times|\cA|}$ with $R (s,a) = \EE_{s'\sim\cP(\cdot|a,s)}[\hat r_{ss'a}^2]$. Substituting in these definitions, we may write
	\begin{eqnarray}
	\label{vari_risk}
	\rho(\lambda) = \langle \hat\lambda,R\rangle - \langle\hat\lambda,r\rangle^2,
	\end{eqnarray}
	which is a quadratic function of the variable $\lambda$. Note that the variance risk $\rho(\lambda)$ is non-convex with respect to $\lambda$. 
	\end{example}

\begin{example}[{\bf Divergence for incorporating priors}]\label{eg:KL}\normalfont
Often in applications, we have access to demonstrations, which can be used to learn a prior on the policy for ensuring baseline performance. Let $\bar\lambda$ be a prior state-action distribution learned from demonstrations. Maintaining baseline performance with respect to this prior, or demonstration distribution, then can be encoded as the Kullback-Liebler (KL) divergence between the normalized distribution $\hat\lambda = (1-\gamma)\lambda$ and the prior $\bar\lambda $ stated as 
	\begin{align}\label{risk_not_r}
	\rho(\lambda) =  \hbox{KL}\left((1-\gamma){\lambda} || \bar\lambda \right)
	\end{align}
	which is substituted into \eqref{prob:dual} to obtain a framework for efficiently incorporating a baseline policy. In some scenarios, existing demonstrations are only state trajectories without revealing the actions taken. Then one may estimate the long-term state-only distribution $\mu$ and define the risk as 
$$		\rho(\lambda) =  \hbox{KL}\left( (1-\gamma) \sum_a {\lambda}_a || \mu \right),$$
which measures the divergence between the marginalized state occupancy distribution and the prior.
In addition to KL, one can also use other convex distances such as Wasserstein, total variation, or even a simple quadratic. 
\end{example}

\section{Stochastic Primal-Dual Policy Gradient}\label{sec:algorithm}
We shift focus to developing an algorithmic solution to the caution-sensitive policy optimization problem \eqref{prob:dual}. While the problem upon first glance appears deterministic, the transition matrices $P_a$ are a priori unknown and we assume the presence of a generative model. Such a generative model is fairly common in control/RL applications where a system simulator is available. For a given state action pair $(s,a)$, the generative model provides the next state $s'$ and the stochastic reward $\hat r_{ss'a}$ according to the unknown transition dynamics.  

Thus, we propose methodologies based on Lagrangian duality together with stochastic approximation. Given the convexity of $\rho$, by virtue of duality, \eqref{prob:dual} admits an equivalent formulation as a saddle point problem:
\begin{equation}
\label{lagrangian_risk1}
\max_{\lambda \in \mathcal{L}} \min_{v\in\mathcal{V}}  L\!\left(\!v,\!\lambda\right)\!=\!\ip{\lambda,\!r}-c\rho(\lambda)+\ip{\xi,\!v}+\!\sum\limits_{a\in\cA}\!\lambda_a^\top\!(\!\gamma P_a-\!I)v,
\end{equation}
where $\mathcal{V}$ should be $\mathbb{R}^{|\cS|}$ in principle. However,  we can later on find a large enough compact set to replace the whole space without loss of optimality. By choosing $\xi$ to satisfy $\xi\geq0$ and $\|\xi\|_1 = 1$, we define the dual feasible set $\mathcal{L}$ as
\begin{equation}
\label{defn:lambda-feasible-region}
\cL:=\{\lambda: \lambda\geq0, \|\lambda\|_1 = (1-\gamma)^{-1}\}.
\end{equation} 
Given distribution $\zeta$ over $\cS\times\cA$, define the stochastic approximation of the risk-neutral component of the Lagrangian:
\begin{equation}
	\label{defn:rand-obj}
	L_{(s,a,s'),\bar{s}}^\zeta (v,\lambda): =
	v_{\bar{s}} + \mathbf{1}_{\{\zeta_{sa}>0\}}\cdot\frac{\lambda_{sa}(\hat r_{ss'a}+\gamma v_{s'} - v_s)}{\zeta_{sa}}，
\end{equation}
where $\bar{s}\sim\cP(\xi)$ is a sample from the discrete distribution defined by probability vector $\xi$. Then by direct computation, when the support of $\zeta$ contains that of $\lambda$, i.e., $\text{supp}(\lambda)\subset\text{supp}(\zeta)$, we may write
\begin{align}\label{eq:lagrangian_expectation}
L(v,\lambda) \!=\! \EE_{(s,a,s')\sim\zeta\times \cP(\cdot|a,s),\bar{s}\sim\xi}\!\!\left[\!L_{(s,a,s'),\bar{s}}^\zeta (v,\lambda)\!\right]\!\!-\!c\rho(\lambda).
\end{align} 
Thus, we view \eqref{lagrangian_risk1} as a stochastic saddle point problem.

\begin{algorithm}[t] 
	\caption{Stochastic Risk-Averse (Cautious) RL}
	\label{algo_one}
	{\bf Input}: Sample size $T$. Parameter $\xi \!=\! \frac{1}{|\cS|}\!\cdot\!\mathbf{1}.$ Stepsizes $\alpha,\beta\!\!>\!\!0$.\! Discount $\gamma\!\in\!(0,1\!)$. Constants $M_1,\!M_2>\!0$, $\delta\!\in\!(0,1\!)$.\\
	{\bf Initialize}: Arbitrary $v^1\in\cV$ and $\lambda^1:=\frac{1}{|\cS||\cA|(1-\gamma)}\cdot\mathbf{1}\in\cL.$\\
	 \textbf{for} {$t=1,2,\cdots,T$}{\\
		     Set $\zeta^t: = (1-\delta)(1-\gamma)\lambda^t + \frac{\delta}{|\cS||\cA|}\mathbf{1}$.\\
		    Sample $(s_t,a_t)\sim\zeta^t$ and  $\bar{s}_t\sim\xi$.\\
		    Generate $s'_t\sim \cP(\cdot|a_t,s_t)$ \& $\hat{r}_{s_ts_t'a_t}$ from generative model.\\
		    Construct  $\hat{\nabla }_v L(\!v^t\!\!,\lambda^t)$ [cf. \eqref{defn:sto-grad-v}] and $\hat{\partial }_\lambda L(\!v^t\!\!,\lambda^t)$ [cf. \eqref{defn:sto-grad-lam}] . \\
			Update $v$ and $\lambda$ as   
			\begin{equation}
			\label{defn:v-update}
			v^{t+1} = \Pi_\cV(v^t - \alpha \hat{\nabla }_v L(v^t,\lambda^t))
			\end{equation}
			and 
			\begin{align}
			\label{defn:lam-update-1}
			\!\!\!\!\!\!\lambda^{t+\!\frac{1}{2}} \!=&\underset{\lambda}{\operatorname{argmax}}\,\, \! \!\langle\hat{\partial}_\lambda L(v^t\!\!,\!\lambda^t),\!\lambda\!-\!\lambda^t\rangle\!
			\\
			&\quad -\! \frac{1}{(1\!-\!\gamma)\beta}KL\big(\!(1-\gamma)\lambda\,\!||\!\,(1-\gamma)\lambda^t\big).\nonumber\\
		   \lambda^{t+1} =& \frac{\lambda^{t+\frac{1}{2}}}{(1-\gamma)\|\lambda^{t+\frac{1}{2}}\|_1}\label{defn:lam-update-2}.
		   \end{align}		   	
		}\\
		{\bf Output:} $\bar \lambda: = \frac{1}{T}\sum_{t=1}^{T}\lambda^t$ and $\bar v: = \frac{1}{T}\sum_{t=1}^{T}v^t$.
\end{algorithm}
We propose variants of stochastic primal-dual method applied to \eqref{lagrangian_risk1}. To obtain the primal descent direction, we note that if $\zeta$ is chosen such that $\text{supp}(\lambda)\subset\text{supp}(\zeta)$, an unbiased estimator of the gradient of $L$ w.r.t. $v\in\mathcal{V}$ is
\begin{eqnarray}
\label{defn:sto-grad-v} 
\hat{\nabla }_v L(v,\lambda) & := & \nabla_v L_{(s,a,s'),\bar{s}}^\zeta (v,\lambda)\\
& = &  \mathbf{e}_{\bar{s}} + \mathbf{1}_{\{\zeta_{sa}>0\}}\cdot\frac{\lambda_{sa}}{\zeta_{sa}}(\gamma \mathbf{e}_{s'} - \mathbf{e}_s),\nonumber
\end{eqnarray}
where $\mathbf{e}_{s}\in\mathbb{R}^{|\cS|}$ is a column vector with only the $s$-th entry equaling to 1 and all other entries being 0. Moreover, a dual subgradient of the instantaneous Lagrangian is given as
\begin{eqnarray}
\label{defn:sto-grad-lam} 
\hat{\partial}_\lambda L(v,\lambda) & := & \!\!\!   \mathbf{1}_{\{\zeta_{sa}>0\}}\cdot\frac{\hat r_{ss'a}+\gamma v_{s'} - v_s-M_1}{\zeta_{sa}}\cdot\mathbf{E}_{s,a}\nonumber
\\& \quad & - c\hat{\partial}\rho(\lambda) - M_2\cdot\mathbf{1},
\end{eqnarray}
where $\mathbf{E}_{s,a}\in\mathbb{R}^{|\cS|\times|\cA|}$ is a matrix with $(s,a)$-th entry equal to 1 and all other entries equal to 0.  $\hat{\partial}\rho(\lambda)$ is an unbiased subgradient estimate of the convex but possibly non-smooth function $\rho$, i.e. $\EE[\hat{\partial}\rho(\lambda)]\in\partial \rho(\lambda)$.  In \eqref{defn:sto-grad-lam}, $M_1$ and $M_2$ are the ``shift" parameters specified in Theorem \ref{theorem:convergence} by the convergence analysis in Section \ref{sec:convergence}.
%
 Note that since the function $\rho$ is often known in practice, a full subgradient $u\in\partial\rho(\lambda)$ may be used instead of an instantaneous approximate $\hat{\partial}\rho(\lambda)$. With appropriately defined shift parameters $M_1,M_2$ in the subgradient estimator, if $\zeta>0$, then the dual subgradient is biased with a constant shift:
$$\EE[\hat{\partial}_\lambda L(v,\lambda)]  \in \partial_\lambda L(v,\lambda) - (M_1 + M_2)\cdot\mathbf{1}.$$

With these estimates for the primal gradient and dual subgradient of the Lagrangian \eqref{eq:lagrangian_expectation}, we propose executing primal-dual stochastic subgradient iteration \citep{chen2016stochastic,chen2018scalable} with the KL divergence in the dual domain. 
The detailed steps are summarized in Algorithm \ref{algo_one}. Employing KL divergence in defining the dual update permits us to leverage the structure of $\lambda$ as a distribution to derive tighter convergence rates, as detailed in Section \ref{sec:convergence}. 


Algorithm \ref{algo_one} provides a model-free method for learning cautious-optimal policies from transition samples. Each primal and dual update can be computed easily based on a single observation. Although Algorithm \ref{algo_one}  is given in the tabular form, its spirit of primal-dual stochastic approximation can be generalized to work with function approximations in the primal and dual spaces as the subject of future work.


\section{Convergence Analysis} \label{sec:convergence}
In this section, we establish the convergence of Algorithm \ref{algo_one} when the caution (risk) $\rho$ in \eqref{prob:dual} is convex in $\lambda$, after which we present extensions of Algorithm \ref{algo_one} to address the non-convex variance risk, with its associated convergence presented thereafter. We provide sample complexity results for finding near-optimal solutions whose dependence on the size of the state and action spaces is tight. 

 Before delving into these details, we state a technical condition on the caution function $\rho$ required for the subsequent analysis, which is that we have access to a first-order oracle providing noisy samples of its subgradient, and that the infinity norm of these samples is bounded. 

%
\begin{assumption}
	\label{assumption:SFO-rho} 
	The caution function $\rho(\lambda)$ is convex but possibly non-smooth, and it has bounded subgradients as
	\begin{align}\sup_{\lambda\in\cL}\sup_{u\in\partial \rho(\lambda)}\|u\|_\infty\leq 	\sigma <\infty. 
	\end{align}
	Further, samples $\hat{\partial}\rho(\lambda)$ of its subgradients  are unbiased and have finite infinity norm:
	\begin{align} 
	\EE[\hat{\partial}\rho(\lambda)] \in \partial \rho(\lambda) \; ,\quad 	\sup_{\lambda\in\cL}\|\hat{\partial}\rho(\lambda)\|_\infty\leq \sigma.
	\end{align}
\end{assumption} 
In our subsequent analysis, we treat $\sigma$ as a known constant. In all of Examples \ref{eg:barrier}-\ref{eg:KL}, the caution function $\rho$ is explicitly known, which yields $\hat{\partial}\rho(\lambda)\in\partial\rho(\lambda)$.  
For an instance, in Example \ref{eg:peak}, if we let $\rho(\lambda) = \sup_{j\in[m]} \langle c_j, \lambda\rangle $, then any subgradient is bounded by $|\hat \partial \rho(\lambda) | \leq \sup_j \|c_j\|_{\infty} = \mathcal{O}(1)$.
For another instance, in Example \ref{eg:KL}, $\rho(\lambda) = KL(\hat\lambda\,||\,\mu)$ for some fixed $\mu$ [cf. \eqref{risk_not_r}], the gradient takes the form
\begin{align}
|\nabla_{\lambda_{sa}}\rho(\lambda)|  &=  \left|(1-\gamma)\left(1+\log\left(\hat\lambda_{sa}/\mu_{sa}\right)\right)\right|\nonumber 
\end{align}
for any $s\in\cS$ and $a\in\cA$. Then, we can ensure Assumption \ref{assumption:SFO-rho} by imposing an elementwise lower bound $\delta_0$ on $\mu$ and $\lambda$ s.t. $\mu\geq\delta_0\cdot\mathbf{1}$ and $\lambda\geq\delta_0\cdot\mathbf{1}$. The constant $\delta_0$ may be chosen extremely small, for instance, $\delta_{0} = \min\{10^{-15}\!\!,|\cS|^{-1}|\cA|^{-1}\!\}\!$. Consequently, we have
$$\sigma \leq \cO\left((1-\gamma)\left(1+\log\left(\delta_0^{-1}\right)\right)\right) = \cO(1).$$
%


\subsection{The Case of Convex Caution Risk}\label{subsec:convex}

In this subsection, we characterize the performance of Algorithm \ref{algo_one} when the caution $\rho$ is convex. 
We begin by noting that the saddle point problem \eqref{lagrangian_risk1} does not specify the feasible region $\cV$ for the variable $v$. However, the convergence necessitates $\cV$ to be a compact set rather than the entire $\mathbb{R}^{|\cS|}$. To disambiguate the domain of $v$, next we derive a bounded region that contains the primal optimizer $v^*$.
\begin{lemma}
	\label{lemma:bound-V}
	If $\xi>0$, then the primal optimizer $v^*$ satisfies 
	\begin{align}
	\|v^*\|_\infty\leq (1-\gamma)^{-1}(1 + c\sigma).
	\end{align}
Therefore, we can define the feasible region $\cV$ to be the compact set
	\begin{align}
	\cV: = \left\{v\in\RR^{|\cS|}: \|v\|_\infty\leq2 \frac{1+c\sigma}{1-\gamma}\right\}.
	\end{align}
\end{lemma}
The proof of Lemma \ref{lemma:bound-V} is provided in Appendix \ref{proof_lemma}.  We note that the factor of 2 is incorporated to simplify the analysis. 
%
%

Subsequently, we analyze the primal-dual convergence of  Algorithm \ref{algo_one} for solving \eqref{lagrangian_risk1} (and the equivalently \eqref{prob:dual}). Before providing the main theorem, we introduce a technical result which defines convergence in terms of a form of duality gap. The duality gap measures the distance of the Lagrangian evaluations to a saddle point as defined by \eqref{lagrangian_risk1}.
\begin{theorem}[{\bf Convergence of duality gap}]
	\label{theorem:convergence}
	For Algorithm \ref{algo_one}, select shift parameters $M_1 = \frac{4(1 + c\sigma)}{1-\gamma}$ and $M_2 = c\sigma$, $\delta \in (0,\frac{1}{2})$, $\beta=\frac{1-\gamma}{1+c \sigma} \sqrt{\frac{\log (|\mathcal{S}||\mathcal{A}|)}{T|\mathcal{S}||\mathcal{A}|}}$, and $\alpha=\sqrt{\frac{|\mathcal{S}|}{T}}\left(1+c \sigma\right)$. Let $\bar \lambda$ and $\bar v$ be the output of Algorithm \ref{algo_one} and let $\lambda^*$ be the optimum. Then for the output of Algorithm \ref{algo_one}, 
	%
	we have 
	\begin{align}\label{eq:sample_complexity}
	 \EE[ L(\bar v,& \lambda^*) - \min_{v\in\cV}L(v, \bar \lambda)]
	 \\
	&\leq  \cO\left(\sqrt{\frac{|\cS||\cA|\log(|\cS||\cA|)}{T}}\cdot\frac{1+2c\sigma}{(1-\gamma)^2}\right).\nonumber
	\end{align}
	As a result, to guarantee $\EE[ L(\bar v,\lambda^*) - \min_{v\in\cV}L(v, \bar \lambda)]\leq \epsilon$, the amount of samples needed is 
	\begin{equation}\label{eq:sample_complexity_convex}
	T = \Theta\left(\frac{|\cS||\cA|\log(|\cS||\cA|)(1+2c\sigma)^2}{(1-\gamma)^4\epsilon^2}\right).
	\end{equation}
\end{theorem}
%
The proof of this Theorem is provided in Appendix \ref{proof_convergence_bound}. 

 We may then use the convergence of duality gap to characterize the sub-optimality and constraint violation attained by the output of Algorithm \ref{algo_one} for the problem \eqref{prob:dual}.
 %

\begin{theorem}[{\bf Convergence to optimal caution-sensitive policies}]
	\label{theorem:duality-gap-meaning}
	Let the parameters $M_1$, $M_2$, $\delta$, $\beta$, and $\alpha$, as defined in Theorem \ref{theorem:convergence},  if $\bar \lambda$ is the output of Algorithm \ref{algo_one} after $T$ iterations, then the constraint violation of the original problem \eqref{prob:dual} satisfies
	\begin{eqnarray}
	\label{thm:duality-gap-meaning-0}
	\begin{cases}
	\bar{\lambda}\geq0, \quad
	\left\|\bar\lambda\right\|_1 = (1-\gamma)^{-1}\vspace{0.4cm}\\
	\left\|\sum_{a\in\cA}(I-\gamma P_a^\top)\bar\lambda_a - \xi\right\|_1\leq \frac{(1-\gamma)\epsilon}{1 + c\sigma}\leq {(1-\gamma)\epsilon}.
	\end{cases} 
	\end{eqnarray}
	Moreover, the sub-optimality of \eqref{prob:dual} is given as
	\begin{eqnarray}
	\label{thm:duality-gap-meaning-0.5}
	\EE[(\langle\lambda^*,r\rangle - c\rho(\lambda^*)) - (\langle\bar\lambda,r\rangle - c\rho(\bar\lambda))]\leq \epsilon
	\end{eqnarray}
\end{theorem}
Eqs.\ \eqref{thm:duality-gap-meaning-0} and \eqref{thm:duality-gap-meaning-0.5} showed the output solution is $\epsilon$-feasible and $\epsilon$-optimal.
Note that $\epsilon$ determines the number of samples $T$ as given in \eqref{eq:sample_complexity_convex}. The proof is provided in Appendix \ref{proof:theorem:duality-gap-meaning}. 

Theorem \ref{theorem:duality-gap-meaning} suggests that to get $\epsilon$-optimal policy and its corresponding state-action distribution, the sample complexity has near-linear dependence (up to logarithmic factors) on the sizes of $\mathcal{S}$ and $\mathcal{A}$. This matches the optimal dependence in the risk-neutral case, see e.g. \citep{chen2018scalable,wang2017primal,wang2017randomized} which proves that Algorithm \ref{algo_one} is sample-efficient.

Further, consider the case where $\rho$ is a KL divergence as in Example \ref{eg:KL}. This $\rho$ acts as a regularization term to keep $\lambda$ close to a prior long-term behavior. In this case, we can show that the primal-dual algorithm enjoys better convergence rates. In particular, we show a tighter KL divergence bound between the estimated $\bar \lambda$ and the optimal state-action distribution.

%
\begin{corollary}[{\bf The case when $\rho$ is a KL divergence}]\label{proposition:breg_distance}
	If $\bar \lambda$ is the output of Algorithm \ref{algo_one}, with parameters $M_1$, $M_2$, $\delta$, $\beta$, $\alpha$ and $T$ chosen as in Theorem \ref{theorem:convergence}, with $\rho(\lambda): = KL\big((1-\gamma)\lambda\,||\,\mu\big)$ is the KL divergence from  given prior $\mu$, we have
$\EE\left[KL\left((1-\gamma)\bar{\lambda}\,||\,(1-\gamma)\lambda^*\right)\right]\leq \frac{\epsilon}{c}$.
\end{corollary}
Corollary \ref{proposition:breg_distance} explicitly gives a dependence of the risk in terms of penalty parameter $c$, which may be made  small as the penalty parameter grows large. Next, we discuss the case where the caution $\rho$ is not necessarily convex.
\begin{algorithm}[t] 
	\caption{A Block Coordinate Ascent (PCA) framework for Policy Optimization with Nonconvex Caution}\label{algo_BCD}
	\textbf{Initialize:} $\lambda^0,\mu^0$.\\
	\textbf{for}{ $k=0,1,...,K-1$}{\\
	Update $\mu^{k+1}$ by solving 
	\begin{equation}
	\label{prob:BCD-mu}
	\!\!\mu^{k+1} \!\!= \arg\max \Phi(\!\lambda^k,\mu)\quad\text{s.t.}\quad \mu\!\geq \!0,\,\,\|\mu\|_1 = 1
	\end{equation}
	with a known closed form solution. \\
	Update $\lambda^{k+1}$ by solving the following to $\epsilon$-sub-optimality
\begin{equation}
\label{prob:BCD-lambda} 
\!\!\max_{\lambda} \,\,\Phi(\!\lambda,\mu^{k-1}) \,\, \text{s.t.} \ \sum\limits_{a\in\cA}\!(I-\gamma P_a^\top) \lambda_a=\xi, \  \lambda\geq 0
\end{equation}
using Algorithm \ref{algo_one}.\\
}
\textbf{Output:} Select $(\lambda^{k^*},\mu^{k^*})$ randomly from $(\lambda^1,\mu^1),...,(\lambda^{K},\mu^{K})$.
\end{algorithm}
\subsection{Extension to Nonconvex Variance Risk} \label{subsec:nonconvex}
In this subsection, we specify the caution as variance as in Example \ref{eg:variance}, which is nonconvex unlike other examples. For this instance, our strategy of addressing the constraints of problem \eqref{prob:dual} via Lagrangian relaxation fails here due to the nonconvexity of the variance in terms of $\lambda$. Therefore, we propose to approximately solve the nonconvex saddle point problem by solving a blockwise convex surrogate problem. 
Consider the following surrogate problem for some $M>0$:
\begin{align} 
\label{eq:alternative_dual_LP}
\!\!\!\!\!\!\max_{\lambda\in\Lambda}\max_{\mu\in U}  \,\,\, \!\!\Phi(\lambda,\mu)\!:=\! \!\ip{\lambda,r}\!-\!c\rho(\lambda,\mu) \!-\! \frac{M}{2}\|\mu\!-\!\hat \lambda\|^2
\end{align}
where $\Lambda:=\big\{\lambda:\sum\limits_{a\in\cA}(I-\gamma P_a^\top) \lambda_a=\xi,  \lambda_a\geq 0, a\in\cA\big\}$,
\begin{align*} 
 \qquad \ \  U:&= \big\{\mu:\mu\geq 0, \|\mu\|_1 = 1\big\}\; ,  \  \hat \lambda = (1-\gamma)\lambda,  \\
  \rho(\lambda,\mu) & =  \langle\hat\lambda,r\rangle^2 - 2\langle \mu,r\rangle\langle\hat{\lambda},r\rangle + \langle\mu,R\rangle.
\end{align*} 
Note that the surrogate problem \eqref{eq:alternative_dual_LP} is not equivalent to the original problem \eqref{prob:dual} when risk $\rho$ is chosen to be variance. However, in this alternative formulation  a quadratic penalty is applied to push distribution $\mu$ towards $\hat{\lambda}$. Observe that when $\mu = \hat \lambda$, we have 
$\rho(\lambda,\mu) = \langle \hat\lambda,R\rangle - \langle\hat\lambda,r\rangle^2$, which equals exactly to the variance function. Therefore, the problem \eqref{eq:alternative_dual_LP} will be  close to the original problem \eqref{prob:dual} when the penalty parameter $M$ is reasonably large. In what follows we propose algorithmic solution to the surrogate problem, assuming that a sufficiently large $M$ is chosen.

The surrogate objective $\Phi$ is strongly concave in $\lambda$ for any fixed $\mu$ and is strongly concave in $\mu$ for any fixed $\lambda$. But $\Phi$ is not jointly concave in $\lambda$ and $\mu$. Therefore, we can employ block-coordinate ascent (BCA) to solve problem \eqref{eq:alternative_dual_LP}. The BCA method alternates between the two steps: First we fix $\lambda$ and optimize the problem over $\mu$ - this subproblem is a projection onto a simplex and has a closed form solution (see \citep{wang2013projection}); Second we fix $\mu$ and optimize over $\lambda$, which is a convex problem and can be solved by using Algorithm \ref{algo_one}. The full scheme is presented in Algorithm \ref{algo_BCD}.
%
%
\begin{figure}[t]
 	\centering
	\hspace{-4mm}
 	\subfigure[!][Reward dist.]{\includegraphics[width=0.34\linewidth,height=4.25cm,]{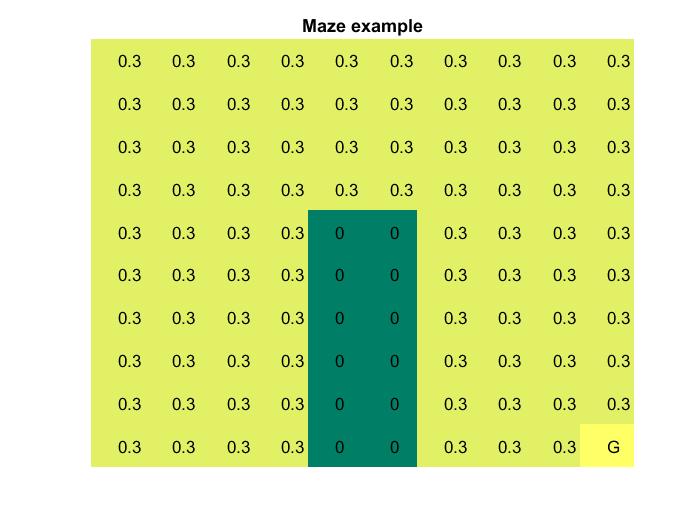}\label{subfig:env010}}\hspace{-3mm}
 	\subfigure[Risk neutral]{\includegraphics[width=0.34\linewidth,height=4.25cm]{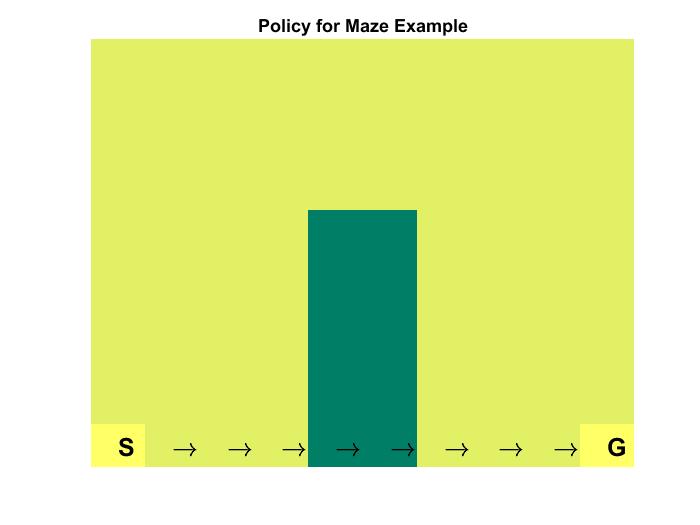}\label{subfig:risk_neutral010}}\hspace{-3mm}
 	\subfigure[Risk averse]{\includegraphics[width=0.34\linewidth, height=4.25cm]{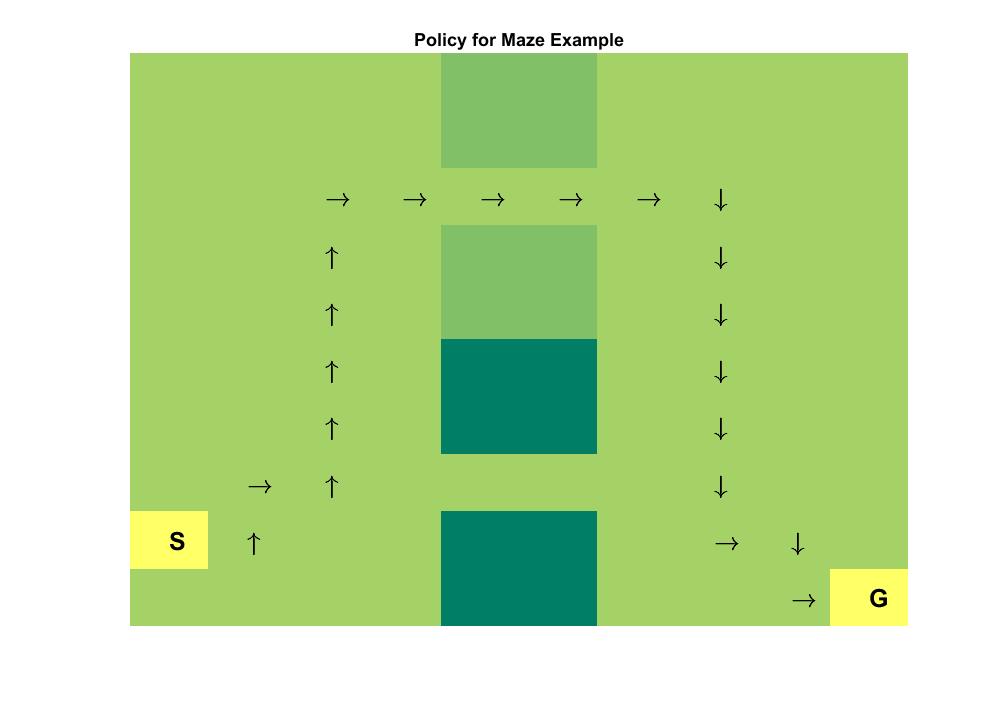}\label{subfig:risk_averse010}}	\vspace{-3mm}
 \caption{Experiment on grid world with variance as the risk. (a) Reward distribution for the Maze environment; (b) Risk neutral and (c) Risk averse  trajectories, respectively, from start to goal. The trajectory resulting from greedily following the risk-averse policy avoids negative reward states.}
 	\label{fig:risk_averse10}\vspace{-3mm}
 \end{figure}
 %
Finally, we establish the sample complexity of Algorithm \ref{algo_BCD} to find a first-order stationary point of \eqref{eq:alternative_dual_LP}.

\begin{theorem}[{\bf Convergence to approximate stationarity}]
	\label{thm:multi-stage}
	Suppose we apply Algorithm \ref{algo_one} to solve the subproblem \eqref{prob:BCD-lambda} with $T$ satisfying 
	$$T = \Theta\left(\frac{|\cS||\cA|\log(|\cS||\cA|)}{(1-\gamma)^4\epsilon^2}\left(1+(1-\gamma)^2(c^2+M^2)\right)\right).$$
	And we solve the subproblem \eqref{prob:BCD-mu} with a closed form solution \citep{wang2013projection}.	Let the number of outer iterations be 
	$$K \geq \frac{\max_{\lambda,\mu}\Phi(\lambda,\mu)-\Phi(\lambda^0,\mu^0)}{\epsilon}.$$
	Then the output $(\lambda^{k^*},\mu^{k^*})$ of Algorithm \ref{algo_BCD} is an approximate-stationary solution to problem \eqref{eq:alternative_dual_LP}, which satisfies
	\begin{align}
	\label{defn:EpsSolu-multi-stage-1}
	\EE\big[&\|\Pi_\Lambda(\nabla_\lambda \Phi(\lambda,\mu))\|^2+\|\Pi_U(\nabla_\mu \Phi(\lambda,\mu))\|^2\big]\nonumber
	\\
	 &\leq \cO\left((1-\gamma)^2\left(\frac{c^2|\cS|^2|\cA|^2}{M} + M\right)\epsilon\right)
	\end{align}
    and 
	\begin{equation}
	\label{defn:EpsSolu-multi-stage-2}
	\begin{cases}
	\EE[\|\sum\limits_{a\in\cA}(I-\gamma P_a^\top) \lambda^{k^*}_a-\xi\|_1] \leq (1-\gamma)\epsilon,  \\
	\lambda^{k^*}\geq 0,\,\,\,\,\, \|\lambda^{k^*}\|_1 = (1-\gamma)^{-1},\\
	\mu^{k^*} \geq 0, \,\,\,\,\,\|\mu^{k^*}\|_1 = 1.
	\end{cases} 
	\end{equation}
\end{theorem}

Eqs.\ \eqref{defn:EpsSolu-multi-stage-1}, \eqref{defn:EpsSolu-multi-stage-2} suggest that both the projected gradient norm and the level of constraint violation are $\cO(\epsilon)$ small. They imply that the output solution is nearly feasible and nearly stationary. See Appendix \ref{Appendix:theorem-non-convex} for the proof of theorem.



\begin{figure}[t]
\hspace{-1mm}
	\subfigure[Objective]{\includegraphics[width=0.33\columnwidth,height=0.2\textheight]{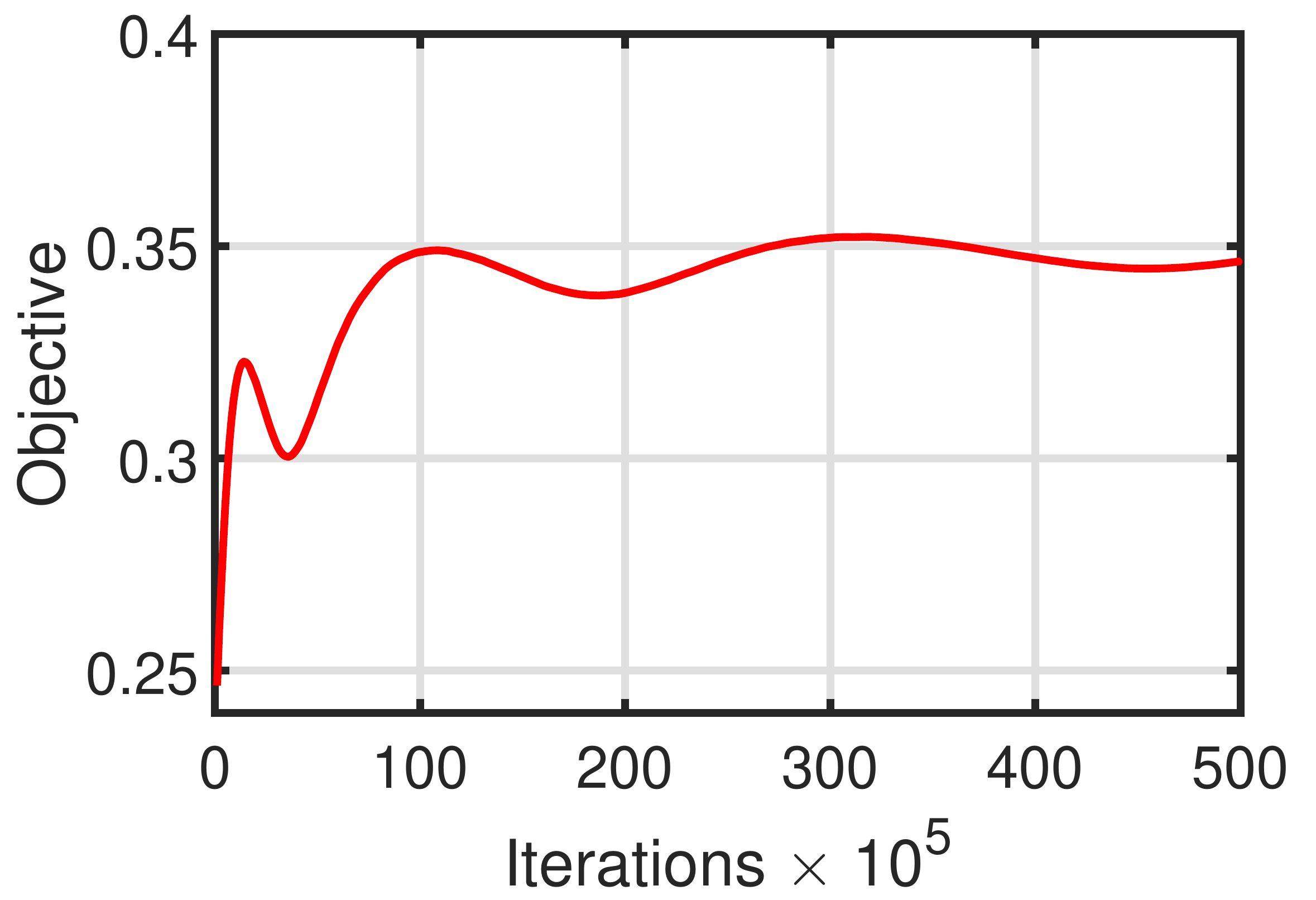}\label{subfig:risk_neutral0010}}\hspace{-2mm}
		\subfigure[Reward mean ]{\includegraphics[width=0.33\columnwidth,height=0.2\textheight]{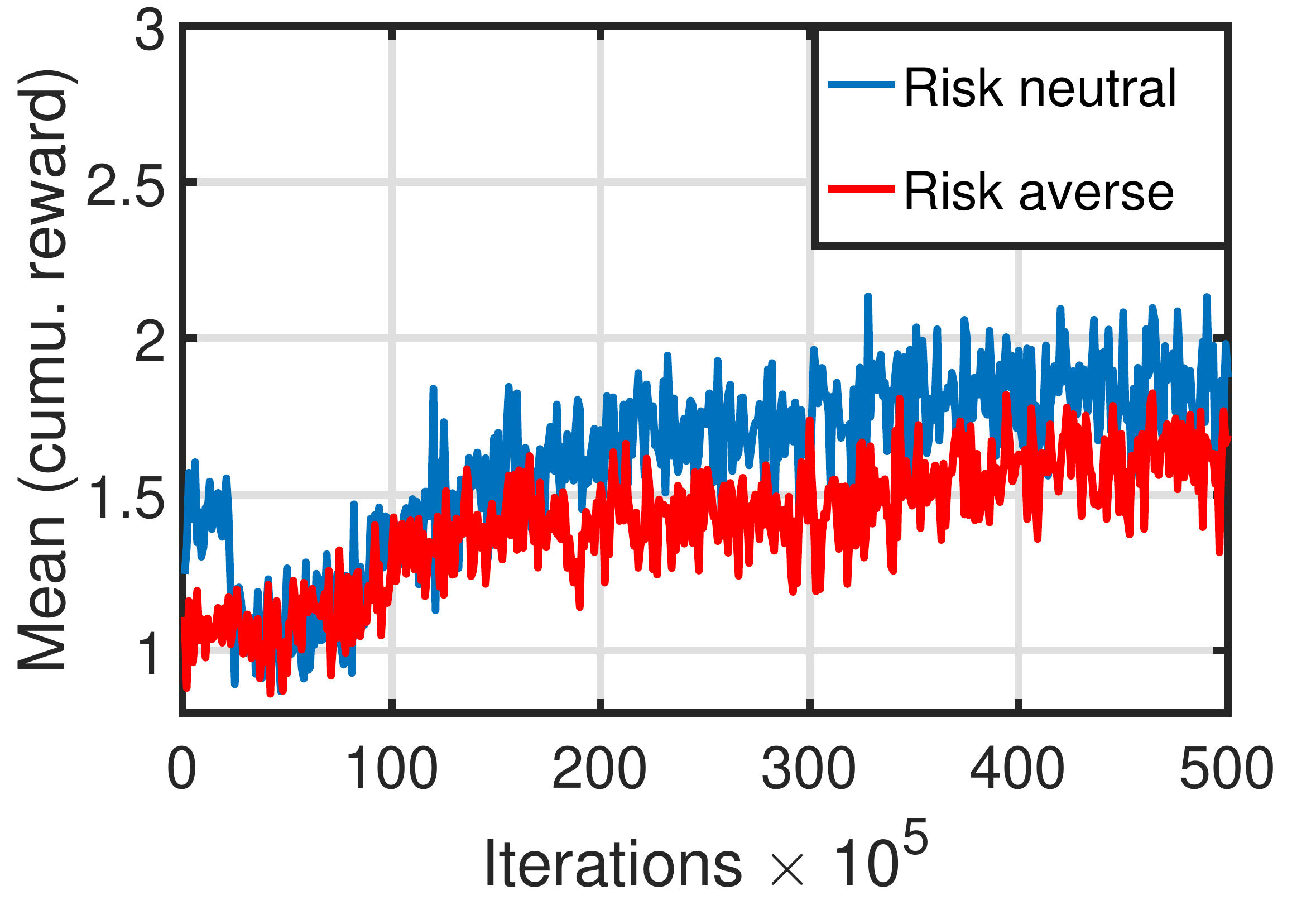}\label{subfig:risk_averse0010}}\hspace{-2mm}
		\subfigure[Reward Variance ]{\includegraphics[width=0.33\columnwidth,height=0.2\textheight]{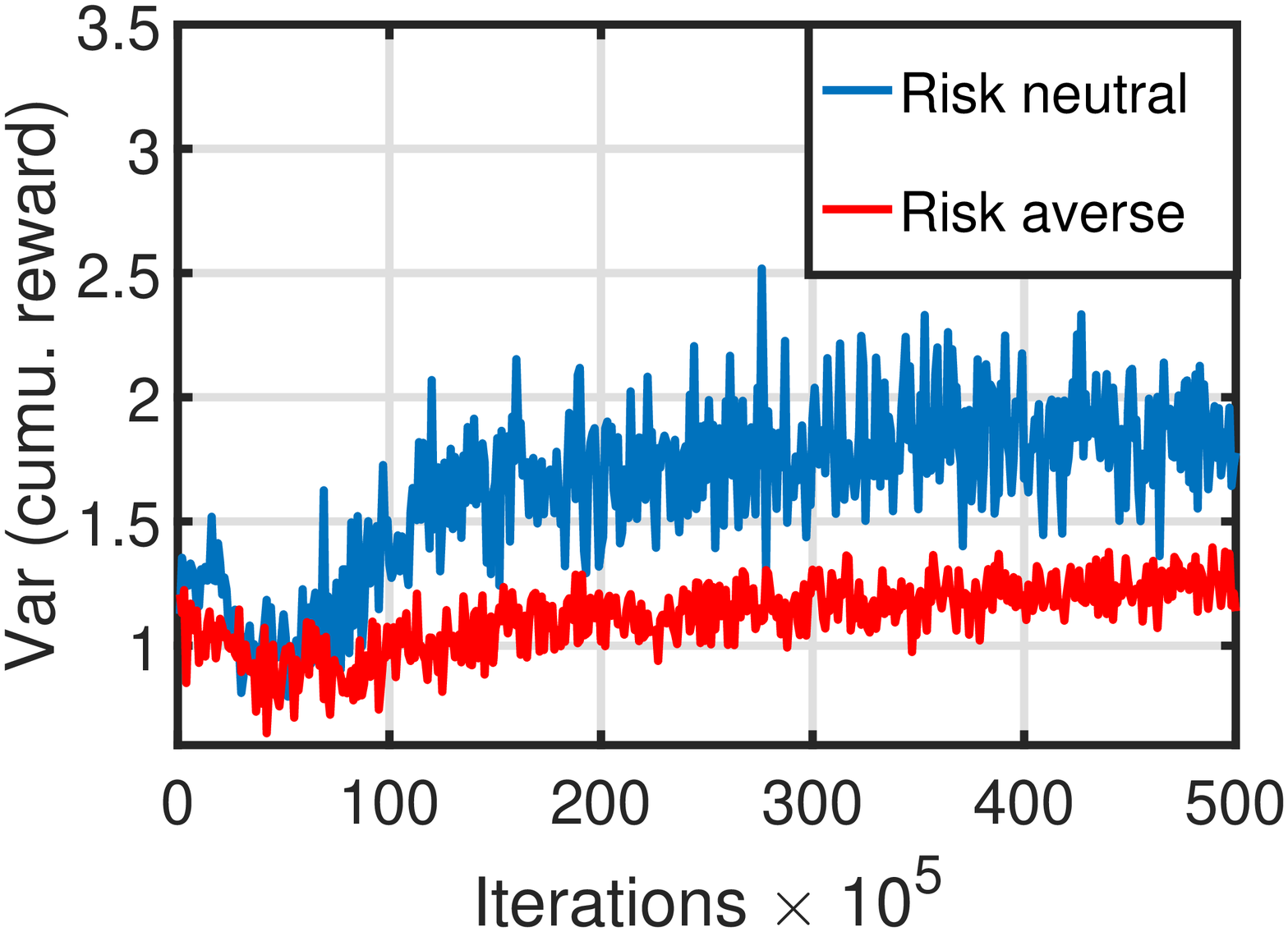}\label{subfig:risk_averse00100}}\vspace{-2mm}
		 \caption{(a) Convergence of the dual objective [cf. \eqref{prob:dual}]; Sample mean return (b) and variance (c) over $100$ simulated trajectories. Observe the expected reward return is comparable while the risk-averse policy attains lower variance, and is thus more reliable. }\vspace{-4mm}
		 	\label{fig:risk}
\end{figure}

%

%

\section{Experimental Results}\label{sec:experiments}
In this section, we experimentally evaluate the proposed technique for incorporating risk or other sources of exogenous information into RL training. In particular,  we consider a setting in which an agent originally learns in the risk-neutral sense of \eqref{main}, i.e., focusing on expected returns. The MDP we focus on is a $10\times 10$ grid with each state permitting for four possible actions (moving $\mathcal{A}:=\{\texttt{up}, \texttt{down}, \texttt{left}, \ \ \text{and}\ \ \texttt{right}\}$). For the transition model, given the direction of the previous action selection, the agent movies in the same direction with probability $p$ and moves in the different direction with probability $1-p$, and moves backwards with null probability. For instance, in a given state action pair $(s,a)$, suppose the action $a$ selected is $\texttt{up}$. Then, the next action will be $\texttt{up}$ with prob $p$ and $\{\texttt{left}, \ \ \text{or}\ \ \texttt{right}\}$ with prob $1-p$, and $\texttt{down}$ with null probability. Overall, this means that the transition matrix has four nonzero sequences of likelihoods along the main diagonal, i.e., it is quad-diagonal. For the experiments, we consider the caution-sensitive formulation presented in Examples \ref{eg:variance} and \ref{eg:KL} which respectively correspond to quantifying risk via the variance and the KL divergence to a previously learned policy which serves as a prior. We append videos (links in the footnote\footnote{\url{https://tinyurl.com/sk4lddb}}\footnote{ \url{https://tinyurl.com/tlcl3m2}}) to the submission which visualize the safety of risk-awareness during training.
 
 \begin{figure}[t]
 	\centering\hspace{-4mm}
 	\subfigure[!][$\mu$]{\includegraphics[width=0.35\linewidth,height=4.25cm]{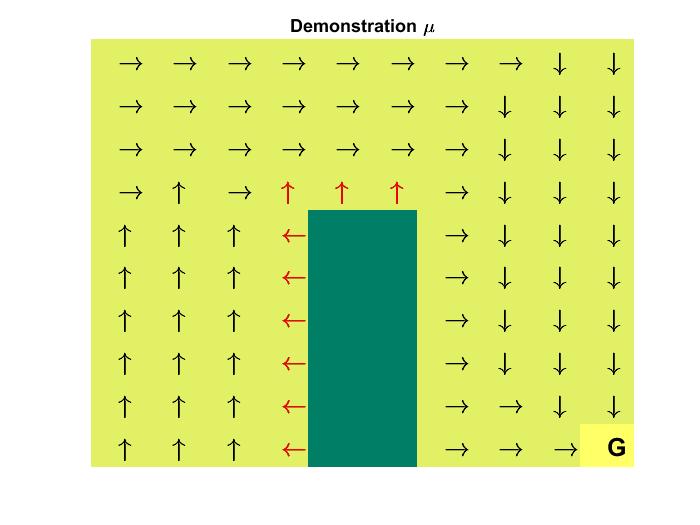}\label{subfig:env}}\hspace{-6mm}
 	\subfigure[Risk neutral]{\includegraphics[width=0.35\linewidth,height=4.25cm]{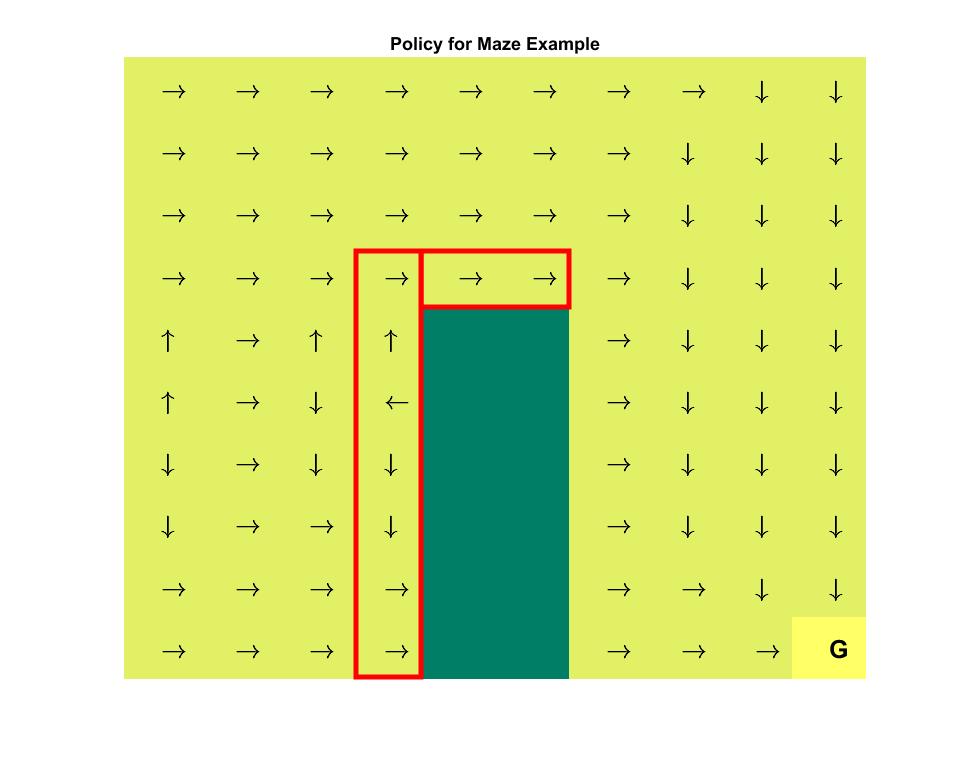}\label{subfig:risk_neutral1}}\hspace{-6mm}
 	\subfigure[Risk averse]{\includegraphics[width=0.35\linewidth, height=4.25cm]{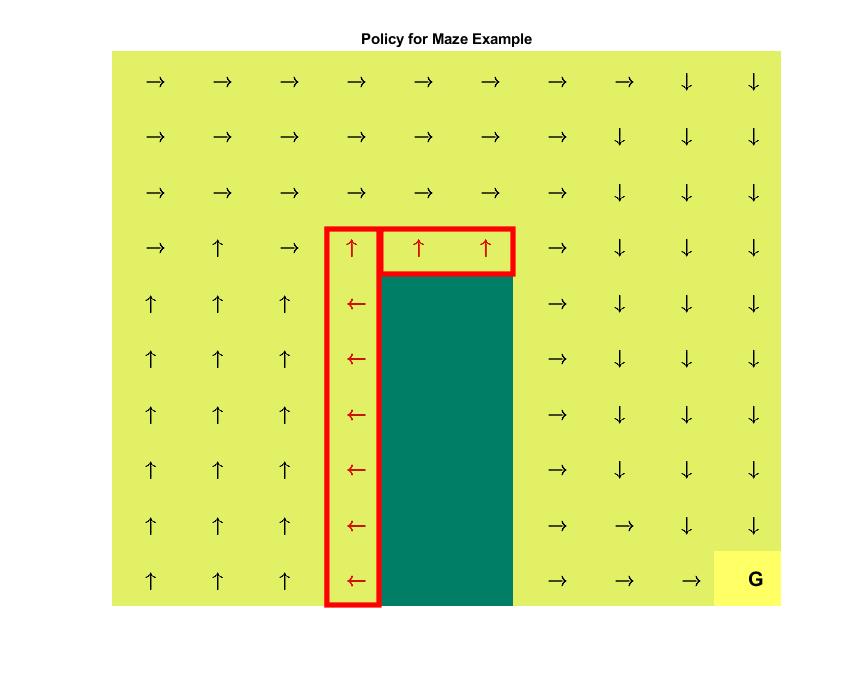}\label{subfig:mnist1}}\vspace{-4mm}
 	\caption{Results for the learning with demonstration $\mu$. We have used KL divergence as the risk function for these results.  (a) The given demonstration, (b) Risk neutral solution, (c) Risk averse solution. Note that incorporating KL divergence yields a policy that avoids unrewarding states (red block in (b) and (c)).}
 	\label{fig:risk_averse1}\vspace{-5mm}
 \end{figure}
\subsection{Variance-Sensitive Policy Optimization}
The variance risk given in Example \ref{eg:variance} characterizes the statistical robustness of the rewards from a policy. 
To evaluate the merit of this definition, consider the maze example with the rewards distribution as described in Fig.~\ref{subfig:env010}. There are two ways to go from start to destination. The reward of dark green areas is more negative than lighter shades of green, and thus it is riskier to be near darker green in terms of the returns of a trajectory. We display a sample path of the Markov chain obtained by solving the variance-sensitive policy optimization problem as Fig.~\ref{subfig:risk_averse010}, whereas the one based on the risk-neutral (classical) formulation is shown in Fig.~\ref{subfig:risk_neutral010}. Clearly, the risk-averse one avoids the dark green areas and collects a sequence of more robust rewards, yet still reaches the goal. The convergence of objective is plotted in Fig. \ref{subfig:risk_neutral0010} for the proposed algorithm. Further, we plot the associated sample mean and variance of the  discounted return over number of training indices in Figs.~\ref{subfig:risk_averse0010} and \ref{subfig:risk_averse00100}, respectively. Observe that the risk-averse policy yields comparable mean reward accumulation with reduced variance, meaning it more reliably reaches the goal without visiting unwanted states whose rewards are negative.
\subsection{Caution as Proximity to a Prior}
When a prior is available in the form of some baseline state-action distribution $\mu$ , KL divergence to the baseline makes sense as a measure of caution [cf. \eqref{risk_not_r}] as stated in Example \ref{eg:KL}. To evaluate this definition, consider the setting where the baseline $\mu$ is a risk-neutral policy (shown in Fig.~\ref{subfig:env}) learned by solving \eqref{dual} with a reward that is highly negative $r=-5$ in the dark green area, strictly positive $r=0.3$ in the light green area, and $r=1$ at the goal in the bottom right  denoted by $\mathsf{G}$ in Fig.~\ref{subfig:env}.  The transition probabilities are defined by $p=0.4$. Then, the resulting risk-neutral policy is used as a baseline policy for a  drifted MDP  whose reward is $r=0$ for the dark green area while identical elsewhere, and whose transition dynamics are defined by likelihood parameter $p=0.6$. The overarching purpose is that although the reward landscape and transition dynamics changed, the ``lessons" of past learning may still be incorporated.

The resulting policy learned from this procedure, as compared with the risk-neutral policy, are visualized in Figures \ref{subfig:risk_neutral1} and \ref{subfig:mnist1}, respectively.  Observe that the policy associated with incorporating past experience in the form of policy $\mu$ has explicitly pushed avoidance of the dark green region, whereas the risk-neutral policy resulting from \eqref{dual} does not. Thus, past (negative) experiences may be incorporated into the learned policy. This hearkens back to psychological experiments on mice: if its food supply is electrified, then a mouse will refuse to eat, even after the electricity is shut off, a form of fear conditioning. Further, we plot the associated discounted return and empirical occupancy of negative reward states with the iteration index of the optimization procedure in Algorithm \ref{algo_one} in Fig.~\ref{fig:risk1}. Overall, then, the incorporation of prior demonstrations results in the faster  learning (see Fig. \ref{subfig:risk_neutral00100}) and reduces the proportion of time spent in unrewarding states as evidenced by Fig. \ref{subfig:risk_averse0010000}.
 \begin{figure}\centering
  	\subfigure[Comparison]{\includegraphics[width=0.48\columnwidth,height=0.22\textheight]{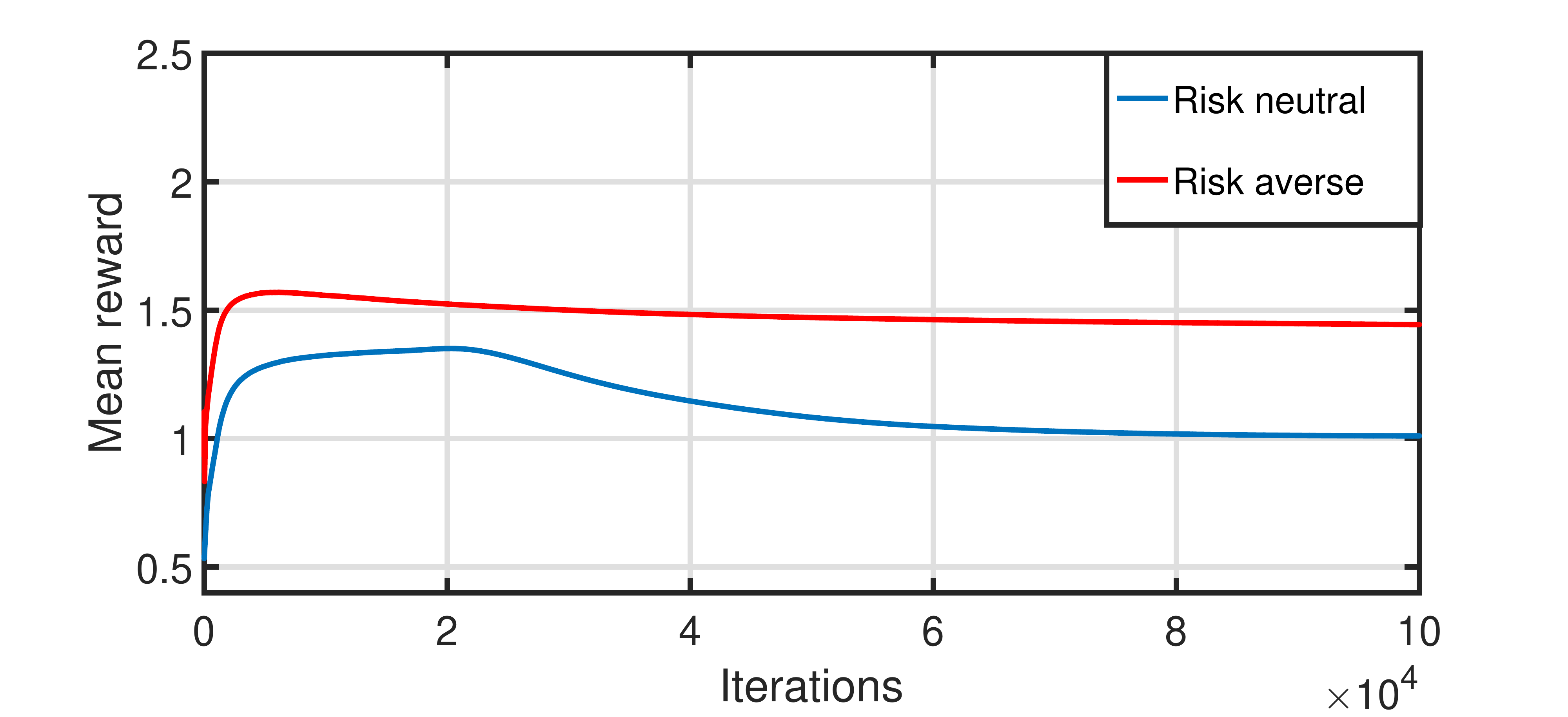}\label{subfig:risk_neutral00100}}
  		\subfigure[\!Time in unrewarding states ]{\includegraphics[width=0.48\columnwidth,height=0.22\textheight]{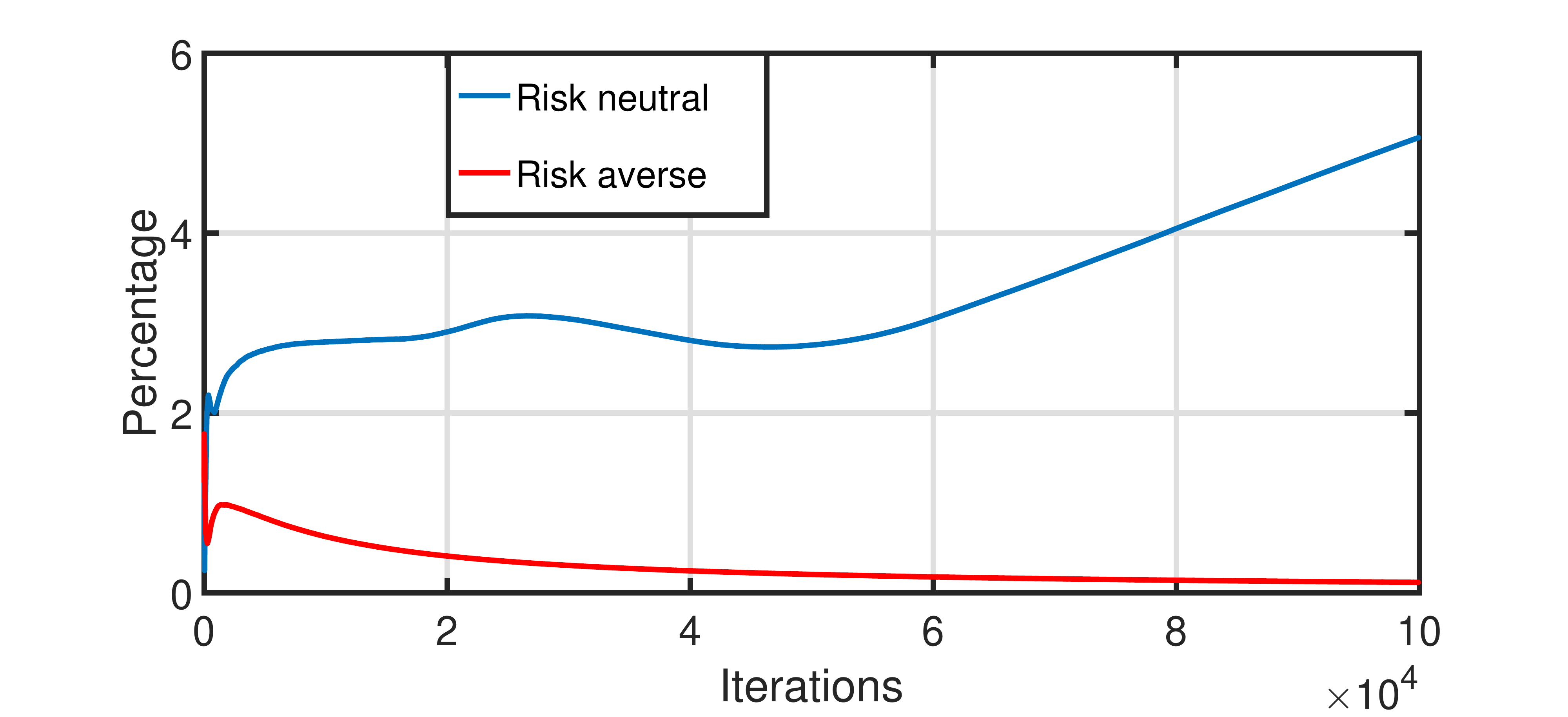}\label{subfig:risk_averse0010000}}\vspace{-3mm}
  		 \caption{We plot the running average of (a) Expected reward return, (b) percentage of time we visit the unrewarding states. Note that the prior demonstration helps in the faster convergence as clear from (a). Further, the KL divergence based risk helps to avoid the visitation of the unrewarding states as clear from the result in (b).}\vspace{-6mm}
  		 	\label{fig:risk1}
  \end{figure}
%
%



\section{Conclusions}\label{sec:conclusions}

In this work, we proposed a new definition of risk named caution which takes as input unnormalized state-action occupancy distributions, motivated by the dual of the LP formulation of the MDP. To solve the resulting risk-aware RL in an online model-free manner, we proposed a variant of stochastic primal-dual method to solve it, whose sample complexity matches optimal dependencies of risk-neutral problem. Experiments illuminated the usefulness of this definition in practice. Future work includes deriving the Bellman equations associated with cautious policy optimization \eqref{prob:dual}, generalizations to continuous spaces, and broadening caution to encapsulate other aspects of decision-making such as inattention and anticipation.

\bibliographystyle{icml2020}
\bibliography{bibliography}

 
\clearpage\newpage\onecolumn
\appendix
\section*{\centering Supplementary Material for ``Cautious Reinforcement Learning via Distributional Risk in the Dual Domain"}

\section{The Physical Meaning of Dual LP}
The dual LP formulation \eqref{dual} has a clear physical meaning. Suppose $\xi\geq0$ and $\|\xi\|_1=1$ is a distribution over the state space $\cS$. Then the following proposition explains the meaning of the dual problem. 
\begin{proposition}
	\label{proposition:dual-meaning}
Suppose the variable $\lambda\in\RR_+^{|\cS|\times|\cA|}$ satisfies the conditions
\begin{equation}
\label{prop:dual-meaning-1}
	\lambda\geq0 \quad \mbox{and}\quad \sum\limits_{a\in\cA}(I-\gamma P_a^\top)\lambda_a=\xi,
\end{equation}
Then $\lambda$ is an  \textbf{unnormalized distribution}, or \textbf{flux}, under the randomized policy $\pi$:
\begin{equation}
\label{prop:dual-meaning-2}
\pi(a|s) = \frac{\lambda_{sa}}{\sum_{a'\in\cA}\lambda_{sa'}}, \,\, \quad\mbox{for}\quad\forall a\in\cA, \forall s\in\cS.
\end{equation}
Furthermore, it satisfies 
\begin{equation}
\label{prop:dual-meaning-3}
\lambda_{sa} = \sum_{t=0}^\infty\gamma^t\cdot\mathbb{P}\bigg(i_t = s, a_t = a\,\,\bigg|\,\,i_0\sim\xi, a_t\sim \pi(\cdot|i_t)\bigg)
\end{equation}
and 
\begin{equation}
\label{prop:dual-meaning-4}
\langle\lambda,r\rangle = \mathbb{E}\left[\sum_{t=0}^\infty \gamma^t r_{i_t i_{t+1} a_t}~\bigg|~ i_0 \sim \xi, a_t\sim \pi(\cdot|i_t)\right].
\end{equation}
\end{proposition}
\begin{proof}
	Under the initial distribution $\xi$ and the randomized policy $\pi:\cS\mapsto\Delta_{|\cA|}$ defined in \eqref{prop:dual-meaning-2}, we define a new initial distribution $\hat \xi$ as 
	$$\hat{\xi}_{sa} = \xi_s\cdot\pi(a|s)\quad\mbox{for}\quad \forall s\in\cS, a\in\cA$$
	as the distribution of the initial state-action pair $(s_0,a_0)$. Therefore the dynamics of the state-action pairs $(s_t,a_t)$ form another Markov chain with transition matrix $\hat P\in\RR_+^{|\cS||\cA|\times|\cA||\cS|}$ defined as 
	$$\hat P_\pi(s,a;s',a') = P_a(s,s')\cdot\pi(a'|s').$$
	First, let us prove that \eqref{prop:dual-meaning-1} is equivalent to \eqref{prop:dual-meaning-3}. 	For the ease of notation, we used the multi-indices. Let us view both $r$ and $\lambda$ as vectors with  $s,a$ being a multi-index.  Note that \eqref{prop:dual-meaning-1} implies that for all $s\in\cS$
	$$\xi_s = \sum_{a'\in\cA} \lambda_{sa'} - \gamma\sum_{a'\in\cA}\sum_{s'\in\cS} P_{a'}(s',s)\lambda_{s'a'}.$$
	Multiplying both sides by $\pi(a|s) = \frac{\lambda_{sa}}{\sum_{a'\in\cA}\lambda_{sa'}}$, we get 
	\begin{eqnarray*}
	\hat{\xi}_{sa} & = & \lambda_{sa} - \gamma\sum_{a'\in\cA}\sum_{s'\in\cS} P_{a'}(s',s)\cdot\pi(a|s)\cdot\lambda_{s'a'}\\
	& = & \lambda_{sa} - \gamma\sum_{a'\in\cA}\sum_{s'\in\cS}\hat P_\pi(s',a';s,a)\lambda_{s'a'}
	\end{eqnarray*}
	for any $s\in\cS, a\in\cA$. If we write this equation in a compact matrix form, we get 
	$$\hat \xi = (I-\gamma \hat{P}_\pi^\top)\lambda.$$
	Note that $\|\gamma \hat{P}_\pi^\top\|_2 \leq \gamma < 1$, we know $(I-\gamma \hat{P}_\pi^\top)^{-1} = \sum_{i=0}^{\infty} \gamma^i (\hat{P}_\pi^i)^\top$. Consequently, 
	\begin{eqnarray*}
	\lambda^\top  =  \hat \xi^\top(I-\gamma \hat{P}_\pi)^{-1}
	 =  \hat \xi^\top  + \gamma \hat \xi^\top\hat P_\pi  + \gamma^2\hat \xi^\top\hat P_\pi^2 + \cdots
	\end{eqnarray*}
	If we write the above equation in an elementwise way, we get \eqref{prop:dual-meaning-3}. Consequently, we also have 
	\begin{eqnarray*} 
	\lambda^\top r & = & \hat \xi^\top r + \gamma \hat \xi^\top\hat P_\pi r + \gamma^2\hat \xi^\top\hat P_\pi^2r + \cdots\\ 
	& = &
	\mathbb{E}\left[\sum_{t=0}^\infty \gamma^t \hat r_{i_t i_{t+1} a_t}~\bigg|~ i_0 \sim \xi, a_t\sim \pi(\cdot|i_t)\right],
	\end{eqnarray*}
which is as stated in \eqref{prop:dual-meaning-4}
\end{proof} 

%
%
%

\section{Proof of Lemma \ref{lemma:bound-V}}\label{proof_lemma}
\begin{proof}
	Consider the min-max saddle point problem,
	\begin{equation}
	\max_{\lambda \geq 0} \min_{v\in\R^{|\cS|}}  L\!\left(\!v,\!\lambda\right)\!=\!\ip{\lambda,\!r}-c\rho(\lambda)+\ip{\xi,\!v}+\!\sum\limits_{a\in\cA}\!\lambda_a^\top\!(\!\gamma P_a-\!I)v,
	\end{equation}
	Then $(\lambda^*,v^*)$ solves this saddle point problem if and only if  
	\begin{equation} 
	\label{lm:KKT-1}
	\lambda^* = \underset{\lambda\geq0}{\operatorname{argmax}}\,\, L(v^*,\lambda)\qquad\mbox{and}\qquad\sum\limits_{a\in\cA}(I-\gamma P_a^\top)\lambda_a^*-\xi=0.
	\end{equation}
	A remark is that, this is also the KKT condition for the original convex problem \eqref{prob:dual}. Due to the concavity of $L(v^*,\lambda)$ for any fixed $v^*$, the condition $\lambda^* = \underset{\lambda\geq0}{\operatorname{argmax}}\,\, L(v^*,\lambda)$ is equivalent to the existence of a subgradient $w^*\in\partial_\lambda L(v^*,\lambda^*)$ s.t.
	\begin{equation}
	\label{lm:KKT-2}
	\langle w^*, \lambda-\lambda^*\rangle\leq0\quad\mbox{for}\quad\forall \lambda\geq0.
	\end{equation}
	If we use $u^*$ to denote the specific subgradient in $\partial \rho(\lambda^*)$ that consists $w^*$.	
	For any fixed $s,a$, we know
	$w^*_{sa} = -(e_s - \gamma P_{as})^\top v^*+ r_{sa} - cu_{sa}^*.$ If we choose $\lambda_{s'a'} = \lambda^*_{s'a'}$ for $\forall (s',a')\neq(s,a)$, \eqref{lm:KKT-2} further implies  
	$$\big((e_s - \gamma P_{as})^\top v^*- r_{sa} + cu^*_{sa}\big)(\lambda_{sa}-\lambda^*_{sa})\geq 0,$$
	where $P_{as}$ is a column vector, with $P_{as}(s') = \cP(s'|a,s)$. Combine this inequality with \eqref{lm:KKT-1}, we can formally write the final optimality condition as follows.
	\begin{equation}
	\label{KKT}\exists u^*\in\partial\rho(\lambda^*)\,\,\,\mathrm{s.t.}\,\,\,
	\begin{cases}
	\sum\limits_{a\in\cA}(I-\gamma P_a^\top)\lambda_a^*=\xi, \quad \lambda^* \geq 0,\\
	\big((e_s - \gamma P_{as})^\top v^*- r_{sa} + cu^*_{sa}\big)(\lambda_{sa}-\lambda^*_{sa})\geq 0,\,\,\,\, \forall s\in\cS, \forall a\in\cA, \forall \lambda_{sa}\geq0.
	\end{cases}
	\end{equation}
	
	By \eqref{prop:dual-meaning-3} of Proposition \ref{proposition:dual-meaning}, we know that $$\sum_{a\in\cA}\lambda^*_{sa}\geq\sum_{a\in\cA}\mathbf{Prob}\big(i_0=s,a_0=a|i_0\sim\xi,a_0\sim\pi(\cdot|i_0)\big) = \xi_s>0\quad \mbox{ for }\quad\forall s\in\cS.$$
	Therefore, for any $s\in\cS$, there exists an $a_s$ such that $\lambda^*_{sa_s}>0.$ Therefore, the second inequality of the optimality condition \eqref{KKT} implies that, 
	$$(e_s - \gamma P_{a_ss})^\top v^*- r_{sa_s} + cu^*_{sa_s} = 0\quad\mbox{ for }\quad \forall s\in\cS.$$ 
	Let us denote $\tilde r := [r_{1a_1},\cdots,r_{|\cS|a_{|\cS|}}]^\top\in\RR^{|\cS|}$, $\tilde u:=[u^*_{1a_1},\cdots,u^*_{|\cS|a_{|\cS|}}]^\top\in\RR^{|\cS|}$ and $\tilde P: = [ P_{a_11},\cdots, P_{a_{|\cS|}{|\cS|}}]\in\RR^{|\cS|\times|\cS|}$. Then we can write
	$$\left(I-\gamma\tilde P^\top\right)v^* = \tilde r - c\tilde u.$$
	As a result, 
	\begin{eqnarray*}
		1 + c\sigma \geq \|\tilde r-c\tilde u\|_\infty
		 =  \|(I-\gamma\tilde P^\top)v^*\|_\infty
		\geq  \|v^*\|_\infty -  \|\gamma\tilde P^\top v^*\|_\infty
		\geq  (1-\gamma)\|v^*\|_\infty,
	\end{eqnarray*}
	which implies the statement of Lemma \ref{lemma:bound-V}. 
\end{proof}

\section{Proof of Theorem \ref{theorem:convergence}}\label{proof_convergence_bound}
\begin{proof} 
To make the proof of this result clearer, we will separate part of the major steps into several different lemmas. 
\begin{lemma}
	\label{lemma:convergence-v}
	Suppose the iterate sequence $\{v^t\}$ is updated according to the rule \eqref{defn:v-update} in Algorithm \ref{algo_one}. Then for any $t$, 
	\begin{align}  
		\label{lm:convergence-1}
		\langle \nabla_v L(v^t,\lambda^t), v^t-v\rangle
		\leq & \frac{1}{2\alpha}(\|v^t-v\|^2 - \|v^{t+1}-v\|^2) + \frac{\alpha}{2}\|\hat{\nabla}_v L(v^t,\lambda^t)\|^2\nonumber 
		\\
		& + \langle\nabla_v L(v^t,\lambda^t)-\hat{\nabla}_v L(v^t,\lambda^t),v^t-v\rangle.
	\end{align}
\end{lemma}
The proof of this lemma is provided in Appendix \ref{appendix:lemma:convergence-v}.
\begin{lemma}
	\label{lemma:convergence-lam}
	Suppose the iterate sequence $\{\lambda^t\}$ is updated according to the rule \eqref{defn:lam-update-1} and \eqref{defn:lam-update-2} in Algorithm \ref{algo_one}. For $\forall t$,  
	\begin{align}
		\label{lm:convergence-3}
		- \langle w^t, \lambda^t-\lambda\rangle 
		\leq & \frac{1}{(1-\gamma)\beta}\big(KL\big((1-\gamma)\lambda\,||\,(1-\gamma)\lambda^t\big) - KL\big((1-\gamma)\lambda\,||\,(1-\gamma)\lambda^{t+1}\big)\big)  \nonumber\\
		& +
		\frac{\beta}{2}\sum_{s,a}\lambda_{sa}^t(\Delta_{sa}^t)^2 + \langle\hat \partial_\lambda L(v^t,\lambda^t)-w^t,\lambda^t-\lambda\rangle,
	\end{align}
    where $w^t := \EE\left[\hat\partial_\lambda L(v^t,\lambda^t)\big|\lambda^t,v^t\right] + (M_1+M_2)\cdot\mathbf{1}\in\partial_\lambda L(v^t,\lambda^t)$ is a subgradient vector.
\end{lemma}
The proof of this lemma is provided in Appendix \ref{appendix:lemma:convergence-lam}. Based on these two lemmas, we start the proof of Theorem \ref{theorem:convergence}. 
Note that by definition, $\bar v = \frac{1}{T}\sum_{t=1}^{T}v^t$ and $\bar \lambda = \frac{1}{T}\sum_{t=1}^{T}\lambda^t$. Define $\bar v^*:=\argmin_{v\in\cV} L(v,\bar{\lambda})$.
Then by the convex-concave structure of $L$ we have 
\begin{eqnarray}
\label{lm:convergence-4}
L(\bar{v},\lambda^*) - L(\bar v^*,\bar\lambda) & \leq & \frac{1}{T} \sum_{t=1}^{T}\left(L(v^t,\lambda^*) - L(\bar v^*,\lambda^t)\right)\\ 
& = & \frac{1}{T} \sum_{t=1}^{T}\left(L(v^t,\lambda^*) - L(v^t,\lambda^t) + L(v^t,\lambda^t) - L(\bar v^*,\lambda^t)\right)\nonumber\\ 
& \leq & \frac{1}{T} \sum_{t=1}^{T}\left(- \langle w^t, \lambda^t-\lambda^*\rangle +\langle \nabla_v L(v^t,\lambda^t), v^t-\bar v^*\rangle\right)\nonumber,
\end{eqnarray}
where the first line applies Jensen's inequality and last line is due to the convexity of $L(\cdot,\lambda^t)$ and the concavity of $L(v^t,\cdot)$. Note that by specifying $v = \bar v^*$ in \eqref{lm:convergence-1} and $\lambda = \lambda^*$ in \eqref{lm:convergence-3}, we can sum up the inequlities \eqref{lm:convergence-1} and \eqref{lm:convergence-3} for $t = 1,...,T$ to yield 
\begin{align}
&\frac{1}{T} \sum_{t=1}^{T}\left(- \langle w^t, \lambda^t-\lambda^*\rangle +\langle \nabla_v L(v^t,\lambda^t), v^t-\bar v^*\rangle\right) \nonumber\\
 \leq &\! \underbrace{\frac{KL\big(\!(1-\!\gamma)\lambda^{\!*} || (1\!-\gamma)\lambda^1\big)}{T(1\!-\gamma)\beta}}_{T_1}\! +\! \underbrace{\frac{\beta}{2T}\!\!\sum_{t=1}^{T}\!\sum_{s,a}\!\lambda_{sa}^t(\!\Delta_{sa}^t)^2}_{T_2} +\! \underbrace{\frac{1}{T}\!\!\sum_{t=1}^{T}\!\langle\hat \partial_\lambda L(\!v^t,\lambda^t)\!-\!w^t,\lambda^t\!-\!\lambda^{\!*}\rangle}_{T_3}\nonumber\\
 & + \underbrace{\frac{\|v^1-\bar v^*\|^2}{2T\alpha}}_{T_4} + \underbrace{\frac{\alpha}{2T}\sum_{t=1}^{T}\|\hat{\nabla}_v L(v^t,\lambda^t)\|^2}_{T_5} + \underbrace{\frac{1}{T}\sum_{t=1}^{T}\langle\nabla_v L(v^t,\lambda^t)-\hat{\nabla}_v L(v^t,\lambda^t),v^t-\bar v^*\rangle}_{T_6}.\nonumber
\end{align}
Substitute this inequality into \eqref{lm:convergence-4} and take the expectation on both sides, we get 
\begin{align}\label{main_all_T}
\EE[L(\bar v,\lambda^*) - \min_{v\in\mathcal{V}}L(v,\bar \lambda)]\leq \sum_{i=1}^{6}\EE[T_i].
\end{align}
For the $\EE[T_i]$'s, the following bounds hold with detailed derivation provided in Appendix \ref{appendix:Tis}:
\begin{eqnarray*}
\EE[T_1]\leq \frac{\log(|\cS||\cA|)}{T(1-\gamma)\beta},\qquad \EE[T_2] \leq \frac{4\beta c^2\sigma^2}{1-\gamma}+ \frac{128\beta|\cS||\cA|(1 + c\sigma)^2}{(1-\gamma)^3},\qquad \EE[T_3] = 0,
\end{eqnarray*}
\begin{equation*}
\EE[T_4] \leq \frac{8|\cS|(1+c\sigma)^2}{T\alpha(1-\gamma)^2}, \qquad\quad\EE[T_5]\leq \frac{27\alpha}{2(1-\gamma)^2},\qquad\quad \EE[T_6]\leq\frac{3\sqrt{3|\cS|}(1 + c\sigma)}{\sqrt{T}(1-\gamma)^2}.
\end{equation*}

Substitute these bounds for $\EE[T_i]$'s  into inequality \eqref{main_all_T} we get 
\begin{align}
\EE[L(\bar v,\lambda^*) - \min_{v\in\cV}L(v, \bar \lambda)]
\leq & \frac{\log(|\cS||\cA|)}{T(1-\gamma)\beta} + \frac{4\beta c^2\sigma^2}{1-\gamma} + \frac{128\beta|\cS||\cA|(1 + c\sigma)^2}{(1-\gamma)^3} \nonumber
\\
& + \frac{8|\cS|(1+c\sigma)^2}{T\alpha(1-\gamma)^2} +\frac{27\alpha}{2(1-\gamma)^2} + \frac{3\sqrt{3|\cS|}(1 + c\sigma)}{\sqrt{T}(1-\gamma)^2}.
\end{align}
If we choose $\beta =\frac{1-\gamma}{1+c\sigma} \sqrt{\frac{\log(|\cS||\cA|)}{T|\cS||\cA|}}$ and  $\alpha = \sqrt{\frac{|\cS|}{T}}(1+c\sigma)$, we have 
\begin{eqnarray*}
	\EE[L(\bar v,\lambda^*) - \min_{v\in\cV}L(v, \bar \lambda)]
	 \leq  \cO\left(\sqrt{\frac{|\cS||\cA|\log(|\cS||\cA|)}{T}}\cdot\frac{1+c\sigma}{(1-\gamma)^2}\right),
\end{eqnarray*}
which completes the proof. 
\end{proof}

\subsection{Proof of Lemma \ref{lemma:convergence-v}}
\label{appendix:lemma:convergence-v}
\begin{proof}
	Consider the update rule of $v$ provided in \eqref{defn:v-update}. For any $v\in\cV$, it holds that 
	\begin{eqnarray*}
		\|v^{t+1}-v\|^2 & = & \|\Pi_\cV(v^t - \alpha\hat \nabla_v L(v^t,\lambda^t)) - v\|^2\\
		& \leq & \|v^t - \alpha\hat \nabla_v L(v^t,\lambda^t) - v\|^2\\
		& = &  \|v^t - v\|^2 + \alpha^2\|\hat \nabla_v L(v^t,\lambda^t) \|^2 - 2\alpha\langle \hat \nabla_v L(v^t,\lambda^t),v^t - v\rangle\\
		& = & \|v^t - v\|^2 + \alpha^2\|\hat \nabla_v L(v^t,\lambda^t) \|^2 - 2\alpha\langle \hat \nabla_v L(v^t,\lambda^t)-\nabla_v L(v^t,\lambda^t)+\nabla_v L(v^t,\lambda^t),v^t - v\rangle.
	\end{eqnarray*}
	Rearranging the above inequality yields
	$$2\alpha\langle\nabla_v L(v^t,\lambda^t),v^t - v\rangle\leq \|v^t - v\|^2-\|v^{t+1}-v\|^2  + \alpha^2\|\hat \nabla_v L(v^t,\lambda^t) \|^2 - 2\alpha\langle \hat \nabla_v L(v^t,\lambda^t)-\nabla_v L(v^t,\lambda^t),v^t - v\rangle.$$
	Deviding both sides by $2\alpha$ proves lemma.
\end{proof}

\subsection{Proof of Lemma \ref{lemma:convergence-lam}}
\label{appendix:lemma:convergence-lam}
\begin{proof}
	Now let us consider the update rule of $\lambda$ given by \eqref{defn:lam-update-1} and \eqref{defn:lam-update-2}. Note that in the subproblem \eqref{defn:lam-update-1}, the problem is separable for each component of $\lambda$ and allows for a closed form solution, i.e., 
	\begin{align}
		\label{lemma:convergence-lam-1}
		\lambda^{t+\frac{1}{2}}_{sa} =& \underset{\lambda_{sa}}{\operatorname{argmax}}\,\,
		\Delta_{sa}^t\lambda_{sa} - \frac{1}{(1-\gamma)\beta}(1-\gamma)\lambda_{sa}\log\left(\frac{(1-\gamma)\lambda_{sa}}{(1-\gamma)\lambda^t_{sa}}\right)
		\\
		=& \lambda^t_{sa}\cdot\exp\{\beta\Delta_{sa}^t\},\nonumber
	\end{align}
	where we denote $\Delta_{sa}^t$ to be the $(s,a)$-th component of $\hat{\partial}_\lambda L(v^t,\lambda^t)$. Then the next iterate is constructed as
	\begin{align}
		\lambda^{t+1} = \frac{\lambda^{t+\frac{1}{2}}}{(1-\gamma)\|\lambda^{t+\frac{1}{2}}\|_1}.\nonumber
	\end{align}
	Or in a more elementary way, we define
	\begin{align}
		\label{lemma:convergence-lam-2}
		\lambda^{t+1}_{sa} = \frac{\lambda^t_{sa}\cdot\exp\{\beta\Delta_{sa}^t\}}{(1-\gamma)\sum_{s',a'}\lambda^t_{s'a'}\cdot\exp\{\beta\Delta_{s'a'}^t\}}.
	\end{align}

It is straightforward that $\lambda^{t+1}\in\cL$. As a result, for any $\lambda\in\cL$, 
\begin{align}
	\label{lm:convergence-2}
	&KL\big((1-\gamma)\lambda\,||\, (1-\gamma)\lambda^{t+1}\big) - KL\big((1-\gamma)\lambda\,||\, (1-\gamma)\lambda^{t}\big)\\
	= & (1-\gamma)\sum_{s\in\cS}\sum_{a\in\cA}\left( \lambda_{sa}\log\left(\frac{\lambda_{sa}}{\lambda^{t+1}_{sa}}\right)  - \lambda_{sa}\log\left(\frac{\lambda_{sa}}{\lambda^{t}_{sa}}\right)\right)\nonumber\\
	= &  (1-\gamma)\sum_{s\in\cS}\sum_{a\in\cA}  \lambda_{sa}\log\left(\frac{\lambda^t_{sa}}{\lambda^{t+1}_{sa}}\right)\nonumber\\
	= & (1-\gamma)\sum_{s\in\cS}\sum_{a\in\cA}  \lambda_{sa}\left(\log\left((1-\gamma)\sum_{s',a'}\lambda^t_{s'a'}\cdot\exp\{\beta\Delta_{s'a'}^t\}\right) - \beta\Delta_{sa}^t\right)\label{KL_first}\\
	= & \log\left((1-\gamma)\sum_{s',a'}\lambda^t_{s'a'}\cdot\exp\{\beta\Delta_{s'a'}^t\}\right) - (1-\gamma)\beta\sum_{s\in\cS}\sum_{a\in\cA}  \lambda_{sa}\Delta_{sa}^t\nonumber\\
	= & \log\left((1-\gamma)\sum_{s,a}\lambda^t_{sa}\cdot\exp\{\beta\Delta_{sa}^t\}\right) - (1-\gamma)\beta\langle\hat{\partial}_\lambda L(v^t,\lambda^t),\lambda\rangle.\label{KL_last}
\end{align}
The equality in \eqref{KL_first} is obtained by using the elementary definition of $\lambda_{sa}^{t+1}$ in \eqref{lemma:convergence-lam-2}; The last equality of \eqref{KL_last} is obtained by applying the definition of $\Delta_{sa}^t$. Note that 
$$\Delta_{sa}^t = \begin{cases}
\frac{\hat r_{s_ts_t'a_t}+\gamma v_{s'_t} - v_{s_t}-M_1}{\zeta^t_{s_ta_t}}  - c\left(\hat{\partial}\rho(\lambda^t)\right)_{s_ta_t} - M_2, & \mbox{ if } (s,a) = (s_t,a_t),\\
- c\left(\hat{\partial}\rho(\lambda^t)\right)_{s_ta_t} - M_2, & \mbox{ if } (s,a) \neq (s_t,a_t).
\end{cases}$$
When we choose $M_1 = 4(1-\gamma)^{-1}(1 + c\sigma)$ and $M_2 = c\sigma$, we can guarantee that $\Delta_{sa}^t\leq 0$ for all $s\in\cS,a\in\cA$. Therefore, by the fact that $e^x\leq 1 + x + \frac{x^2}{2}$ for all $x\leq0$ and $\log(1+x)\leq x$ for all $x>-1$, we have 
\begin{align}
	\log\left((1-\gamma)\sum_{s,a}\lambda^t_{sa}\cdot\exp\{\beta\Delta_{sa}^t\}\right)
	\leq & \log\left((1-\gamma)\sum_{s,a}\lambda^t_{sa}\cdot\big(1+\beta\Delta_{sa}^t + \frac{\beta^2}{2}(\Delta_{sa}^t)^2\big)\right) \nonumber\\
	= & \log\left(  1+  (1-\gamma)\beta\langle\hat{\partial}_\lambda L(v^t,\lambda^t),\lambda^t\rangle+\frac{(1-\gamma)\beta^2}{2}\sum_{s,a}\lambda^t_{sa}(\Delta_{sa}^t)^2\right)\nonumber\\
	\leq &  (1-\gamma)\beta\langle\hat{\partial}_\lambda L(v^t,\lambda^t),\lambda^t\rangle+\frac{(1-\gamma)\beta^2}{2}\sum_{s,a}\lambda^t_{sa}(\Delta_{sa}^t)^2.\label{upper-bound}
\end{align}
Utilizing the  upper bound of \eqref{upper-bound} into the right hand side of \eqref{lm:convergence-2} results in
\begin{eqnarray*}
&&KL\big((1-\gamma)\lambda\,||\, (1-\gamma)\lambda^{t+1}\big) - KL\big((1-\gamma)\lambda\,||\, (1-\gamma)\lambda^{t}\big) \\
&\leq& \frac{(1-\gamma)\beta^2}{2}\sum_{s,a}\lambda^t_{sa}(\Delta_{sa}^t)^2 + (1-\gamma)\beta\langle\hat{\partial}_\lambda L(v^t,\lambda^t)-w^t+w^t,\lambda^t-\lambda\rangle.
\end{eqnarray*}
Rearranging the terms and deviding both sides by $(1-\gamma)\beta$ proves this lemma. 
\end{proof}

\subsection{Bounding the $\EE[T_i]$'s}
\label{appendix:Tis}
Step 1. Bounding $\EE[T_1]$. Note that $\lambda^1 = \frac{\mathbf{1}}{(1-\gamma)|\cS||\cA|}$, we know
\begin{eqnarray}
\label{lm:convergence-T1}
\EE[T_1] & = & \frac{1}{T(1-\gamma)\beta}\sum_{s,a}(1-\gamma)\lambda_{sa}^*\left(\log(\lambda_{sa}^*) - \log(|\cS|^{-1}|\cA|^{-1})\right)\\
& \leq & \frac{1}{T(1-\gamma)\beta}\sum_{s,a}(1-\gamma)\lambda_{sa}^*\log(|\cS||\cA|)\nonumber\\
& = & \frac{\log(|\cS||\cA|)}{T(1-\gamma)\beta}.\nonumber
\end{eqnarray}

Step 2. Bounding $\EE[T_2]$. For each $t$, we have
\begin{eqnarray*}
	\EE\left[\sum_{s,a}\lambda_{sa}^t(\Delta_{sa}^t)^2\big|v_t,\lambda_t\right]
	&  = & \EE_{s_t,a_t}\left[\sum_{s,a}\lambda_{sa}^t\left(\frac{\hat r_{ss'a}+\gamma v_{s'} - v_{s}-M_1}{\zeta^t_{sa}}\cdot\mathbf{1}_{(s,a) = (s_t,a_t)}  - c\left(\hat{\partial}\rho(\lambda^t)\right)_{sa} - M_2\right)^2\bigg|v_t,\lambda_t\right]\\
	& \leq & 2\EE_{s_t,a_t}\left[\sum_{s,a}\lambda_{sa}^t\left( c\left(\hat{\partial}\rho(\lambda^t)\right)_{sa} + M_2\right)^2+\lambda^t_{s_t,a_t}\left(\frac{\hat r_{s_ts'_ta_t}+\gamma v_{s'_t} - v_{s_t}-M_1}{\zeta^t_{s_ta_t}}\right)^2\bigg|v_t,\lambda_t\right]\\
	& \leq & 8(1-\gamma)^{-1}c^2\sigma^2 + 2 \sum_{s,a}\lambda_{sa}^t\zeta_{sa}^t\left(\frac{\hat r_{ss'a}+\gamma v_{s'} - v_{s}-M_1}{\zeta^t_{sa}}\right)^2\\
	& \leq & 8(1-\gamma)^{-1}c^2\sigma^2 + 2 \sum_{s,a}\frac{\lambda_{sa}^t\left(\hat r_{ss'a}+\gamma v_{s'} - v_{s}-M_1\right)^2}{(1-\delta)(1-\gamma)\lambda_{sa}^t+\frac{\delta}{|\cS||\cA|}}\\
	& \leq & 8(1-\gamma)^{-1}c^2\sigma^2 + 2 \sum_{s,a}\frac{64\lambda_{sa}^t(1-\gamma)^{-2}(1 + c\sigma)^2}{(1-\delta)(1-\gamma)\lambda_{sa}^t+\frac{\delta}{|\cS||\cA|}}\\
	& \leq & 8(1-\gamma)^{-1}c^2\sigma^2 + \frac{128|\cS||\cA|(1 + c\sigma)^2}{(1-\delta)(1-\gamma)^3}\\
	& \leq & 8(1-\gamma)^{-1}c^2\sigma^2 + \frac{256|\cS||\cA|(1 + c\sigma)^2}{(1-\gamma)^3}.
\end{eqnarray*}
The second row follows the definition of $\Delta_{sa}^t$; The 4-th row is due to the assumption that $\|\hat{\partial}\rho\|_\infty\leq\sigma$; In the 5-th we substitute the definition of $\zeta^t_{sa}$ provided in Algorithm \ref{algo_one}; In the 6-th row we substitute the detailed value of $M_1$;  The 8-th row is because $\delta\in(0,\frac{1}{2})$. As a result, we have
\begin{eqnarray}
\label{lm:convergence-T2}
\EE[T_2] & = &\frac{\beta}{2T}\sum_{t=1}^{T} \EE\left[\sum_{s,a}\lambda_{sa}^t(\Delta_{sa}^t)^2\right]\leq \frac{4\beta c^2\sigma^2}{1-\gamma} + \frac{128\beta|\cS||\cA|(1 + c\sigma)^2}{(1-\gamma)^3}.
\end{eqnarray}

Step 3. Bounding $\EE[T_3]$, because $\lambda^*$ is a constant, for each $t$, we have 
\begin{eqnarray*}
	\EE[\langle\hat \partial_\lambda L(v^t,\lambda^t)-w^t,\lambda^t-\lambda\rangle|v^t,\lambda^t]  =  -\langle(M_1+M_2)\cdot\mathbf{1},\lambda^t-\lambda^*\rangle=0,
\end{eqnarray*}
where we have applied the fact that $\sum_{s,a}\lambda^t_{sa} = \sum_{s,a}\lambda^*_{sa}$, and $w^t = \EE[\hat \partial_\lambda L(v^t,\lambda^t)|v^t,\lambda^t] + (M_1+M_2)\cdot\mathbf{1}$
when $\zeta^t>0$. As a result,
\begin{equation}
\label{lm:convergence-T3}
\EE[T_3] = \frac{1}{T}\sum_{t=1}^{T}\EE\left[\langle\hat \partial_\lambda L(v^t,\lambda^t)-w^t,\lambda^t-\lambda^*\rangle\right] = 0.
\end{equation}

Step 4. Bounding $\EE[T_4]$, we have 
\begin{equation}
\label{lm:convergence-T4}
\EE[T_4] = \frac{1}{2T\alpha}\EE\left[\|v^1-\bar{v}^*\|^2\right] \leq \frac{8|\cS|(1+c\sigma)^2}{T\alpha(1-\gamma)^2}.
\end{equation}

Step 5. Bounding $\EE[T_5]$, applying the expression \eqref{defn:sto-grad-v} yields
\begin{eqnarray*}
	\EE\left[\|\hat{\nabla}_v L(v^t,\lambda^t)\|^2\big|v^t,\lambda^t\right] & = &\EE_{s_t,a_t,s_t',\bar{s}_t}\left[\big\|\mathbf{e}_{\bar{s}_t} + \frac{\lambda^t_{s_ta_t}}{\zeta^t_{s_ta_t}}(\gamma \mathbf{e}_{s'_t} - \mathbf{e}_{s_t})\big\|^2\bigg|v_t,\lambda^t\right]\\
	& = & \EE_{s_t,a_t,s_t',\bar{s}_t}\left[\big\|\mathbf{e}_{\bar{s}_t} + \frac{\lambda^t_{s_ta_t}}{(1-\delta)(1-\gamma)\lambda_{s_ta_t}^t+\frac{\delta}{|\cS||\cA|}}(\gamma \mathbf{e}_{s'_t} - \mathbf{e}_{s_t})\big\|^2\bigg|v_t,\lambda^t\right]\\
	& \leq & \EE_{s_t,a_t,s_t',\bar{s}_t}\left[3 + \frac{3\gamma^2 + 3}{(1-\delta)^2(1-\gamma)^2}\bigg|v_t,\lambda^t\right]\\
	& \leq & \frac{27}{(1-\gamma)^2}.
\end{eqnarray*}
Consequently, 
\begin{eqnarray}
\label{lm:convergence-T5}
\EE[T_5]  \,\,= \,\, \frac{\alpha}{2T}\sum_{t=1}^{T}\EE\left[\|\hat{\nabla}_v L(v^t,\lambda^t)\|^2\right]\,\, \leq \,\,\frac{27\alpha}{2(1-\gamma)^2}.
\end{eqnarray}

Step 6. Bounding $\EE[T_6]$. Because $\bar v^*$ is a random variable dependent on $\hat{\nabla}_vL(v^t,\lambda^t)$ we will need the following proposition.
\begin{proposition}[\citep{bach2019universal}]
	Let $\mathcal{Z}\subseteq \mathbb{
		R}^d$ be a convex set and $w:\mathcal{Z}\rightarrow \mathbb{R}$ be a $1$ strongly convex function with respect to norm $\norm{\cdot}$ over $\mathcal{Z}$. With the assumption that for all $x\in\mathcal{Z}$ we have $w(x)-\min_{x\in\mathcal{Z}}w(x)\leq \frac{1}{2} D^2$, then for any martingale difference sequence $\{Z_k\}_{k=1}^K\in\mathbb{R}^d$ and any random vector $z\in\mathcal{Z}$, it holds that 
	\begin{align*}
		\Ex{\sum\limits_{k=1}^{K}\ip{Z_k,x}}\leq \frac{D}{2}\sqrt{\sum\limits_{k=1}^{K}\Ex{\norm{Z_k}_*^2}},
	\end{align*}
	where $\norm{\cdot}_*$ denotes the dual norm of $\norm{\cdot}$.
\end{proposition}    
With this proposition, and note that  $\EE\left[\langle\hat\nabla_v L(v^t,\lambda^t) \big| v^t,\lambda^t\right] = \nabla_v L(v^t,\lambda^t)$, we have
\begin{eqnarray}
\label{lm:convergence-T6}
\EE[T_6]& = & \frac{1}{T}\sum_{t=1}^{T}\EE\left[\langle\nabla_v L(v^t,\lambda^t)-\hat{\nabla}_v L(v^t,\lambda^t),v^t-\bar v^*\rangle\right]\\
& = & \frac{1}{T}\sum_{t=1}^{T}\EE\left[\langle\nabla_v L(v^t,\lambda^t)-\hat{\nabla}_v L(v^t,\lambda^t),\bar v^*\rangle\right]\nonumber\\
& \leq & \frac{\sqrt{|\cS|}(1 + c\sigma)}{T(1-\gamma)}\sqrt{\sum_{t=1}^{T}\EE\left[\|\nabla_v L(v^t,\lambda^t)-\hat{\nabla}_v L(v^t,\lambda^t)\|^2\right]}\nonumber\\
& \leq & \frac{\sqrt{|\cS|}(1 + c\sigma)}{T(1-\gamma)}\sqrt{\sum_{t=1}^{T}\EE\left[\|\hat{\nabla}_v L(v^t,\lambda^t)\|^2\right]}\nonumber\\
& \leq & \frac{\sqrt{|\cS|}(1 + c\sigma)}{T(1-\gamma)}\sqrt{\frac{2T}{\alpha}\EE[T_5]}\nonumber\\
& \leq & \frac{3\sqrt{3|\cS|}(1 + c\sigma)}{\sqrt{T}(1-\gamma)^2}.\nonumber
\end{eqnarray}

\section{Proof of Theorem \ref{theorem:duality-gap-meaning}}\label{proof:theorem:duality-gap-meaning}
\begin{proof}
	The first row of \eqref{thm:duality-gap-meaning-0} is directly satisfied due to the feasibility of $\bar{\lambda}\in\cL$. Now we prove the second row of \eqref{thm:duality-gap-meaning-0}. When the parameters are chosen according to Theorem \ref{theorem:convergence}, we know
	\begin{eqnarray}
	\label{thm:duality-gap-meaning-1}
	\epsilon \geq \EE[L(\bar v ,\lambda^*) - \min_{v\in\mathcal{V}}L(v,\bar{\lambda})].
	\end{eqnarray}
	For the ease of notation, denote $C := (1-\gamma)^{-1}(1 + c\sigma).$ Then substitute the details of $L$ we get 
	\begin{eqnarray}
	\label{thm:duality-gap-meaning-2}
	\min_{v\in\mathcal{V}}L(v,\bar{\lambda}) & = & \min_{\|v\|_\infty\leq2C} \ip{\bar\lambda,r}-c\rho(\bar\lambda)+\ip{\xi,v}+\sum\limits_{a\in\cA}\bar\lambda_a(\gamma P_a-I)v\\
	& = & \ip{\bar\lambda,r}-c\rho(\bar\lambda) - 2C\bigg\|\sum_{a\in\cA}(I-\gamma P_a^\top)\bar\lambda_a - \xi\bigg\|_1\nonumber.
	\end{eqnarray}
	By the feasibility of $\lambda^*$, namely, $\sum_{a\in\cA}(I-\gamma P_a^\top)\lambda^*_a - \xi = 0$, we have 
	\begin{eqnarray}
	\label{thm:duality-gap-meaning-3}
	L(\bar v,\lambda^*) =  \ip{\lambda^*,r}-c\rho(\lambda^*)+\ip{\xi,\bar v}+\sum\limits_{a\in\cA}(\lambda^*_a)^\top(\gamma P_a-I)\bar v = \ip{\lambda^*,r}-c\rho(\lambda^*).
	\end{eqnarray}
	Substituting \eqref{thm:duality-gap-meaning-2} and \eqref{thm:duality-gap-meaning-3} into \eqref{thm:duality-gap-meaning-1} yields
	\begin{eqnarray}
	\label{thm:duality-gap-meaning-4}
	\EE\left[\left(\ip{\lambda^*,r}-c\rho(\lambda^*)\right) - \left(\ip{\bar\lambda,r}-c\rho(\bar\lambda)\right)+ 2C\bigg\|\sum_{a\in\cA}(I-\gamma P_a^\top)\bar\lambda_a - \xi\bigg\|_1\right]\leq \epsilon.
	\end{eqnarray}
	Actually, this inequality has already proved the bound \eqref{thm:duality-gap-meaning-0.5} in terms of the objective value of problem \eqref{prob:dual}. Also, by the feasibility of $\lambda^*$, the convexity of $\rho$ and the optimality condition \eqref{KKT}, we have 
	\begin{eqnarray*}
	 &&\left(\ip{\lambda^*,r}- c\rho(\lambda^*)\right) - \left(\ip{\bar\lambda,r}-c\rho(\bar\lambda)\right) + \big\langle v^*,\sum_{a\in\cA}(I-\gamma P_a^\top)\bar{\lambda}_a - \xi\big\rangle\nonumber\\
	& = & \left(\ip{\lambda^*,r}-c\rho(\lambda^*)\right) - \left(\ip{\bar\lambda,r}-c\rho(\bar\lambda)\right) + \big\langle v^*,\sum_{a\in\cA}(I-\gamma P_a^\top)\bar{\lambda}_a - \sum_{a\in\cA}(I-\gamma P_a^\top)\lambda^*_a\big\rangle\nonumber\\
	& \geq & \sum_{a\in\cA}\big\langle (I-\gamma P_a)v^* -r_a + cu^*_a,\bar\lambda_a-\lambda^*_a\big\rangle \nonumber\\
	& \geq & 0,
	\end{eqnarray*}
    where $u^*\in\partial\rho(\lambda^*)$ is defined in \eqref{KKT}, and $u^*_a := [u^*_{1a},...,u^*_{|\cS|a}]^\top$ is  column vector. Immediately, this implies 
	\begin{eqnarray}
	\label{thm:duality-gap-meaning-5}
	\left(\ip{\lambda^*,r}-c\rho(\lambda^*)\right) - \left(\ip{\bar\lambda,r}-c\rho(\bar\lambda)\right) & \geq & - \big\langle v^*,\sum_{a\in\cA}(I-\gamma P_a^\top)\bar{\lambda}_a - \xi\big\rangle\\
	& \geq & -\|v^*\|_\infty\big\|\sum_{a\in\cA}(I-\gamma P_a^\top)\bar{\lambda}_a - \xi\big\|_1\nonumber\\
	& \geq &-C \big\|\sum_{a\in\cA}(I-\gamma P_a^\top)\bar{\lambda}_a - \xi\big\|_1\nonumber.
	\end{eqnarray}
	where we used the fact that $\|v^*\|_\infty\leq C$ proved in Lemma \ref{lemma:bound-V}. Substitute \eqref{thm:duality-gap-meaning-5} into \eqref{thm:duality-gap-meaning-4} gives 
	$$\EE\big[C \big\|\sum_{a\in\cA}(I-\gamma P_a^\top)\bar{\lambda}_a - \xi\big\|_1\big]\leq \epsilon.$$
	Divide both sides by $C = (1-\gamma)^{-1}(1+c\sigma)$ proves inequality \eqref{thm:duality-gap-meaning-0}.
\end{proof}
\section{Proof of Proposition \ref{proposition:breg_distance}} \label{proof:proposition:breg_distance} 
\begin{proof}
	Let $v^*$ be the optimal for the saddle point problem \eqref{lagrangian_risk1}. Then the duality gap also guarantees that 
	$$\EE[L(\bar v ,\lambda^*) - L(v^*,\bar{\lambda})]\leq \epsilon.$$
	Substituting the detailed form of $L(\cdot,\cdot)$ we get $\EE[T_1+T_2]\leq \epsilon$ where
	\begin{equation}
	\label{prop:breg-1}
	T_1 =  \ip{\xi,\bar v}- c\rho(\lambda^*,r)-[\ip{\xi,v^*}- c\rho(\bar \lambda,r)]
	\end{equation}
	and 
	\begin{equation*}
	T_2=\sum\limits_{a\in\cA}\ip{\bar \lambda_a,(I-\gamma P_a)\bar v-r_a} -\sum\limits_{a\in\cA}\ip{\lambda_a^*,(I-\gamma P_a)v^*-r_a}.
	\end{equation*}
	Note that for $\lambda^*$, $\sum_{a\in\cA}(I-\gamma P_a^\top)\lambda^*_a = \xi$. Consequently, 
	\begin{align}
	T_2	 = &\sum\limits_{a\in\cA}\ip{(I-\gamma P_a) v^*-r_a,\bar \lambda_a-\lambda_a^*} +\sum\limits_{a\in\cA}\ip{ \lambda_a^*,(I-\gamma P_a)(v^*-\bar v)}\nonumber
	\\
	=&\sum\limits_{a\in\cA}\ip{(I-\gamma P_a)v^*-r_a,\bar \lambda_a-\lambda_a^*} +\xi^\top(v^*-\bar v).\nonumber
	\end{align}  
	Let us define $D(\bar \lambda, \lambda^*) := \rho(\bar{\lambda})-\rho(\lambda^*)-\langle\nabla\rho(\lambda^*),\bar{\lambda}-\lambda^*\rangle\geq0$. Combine this equality for $T_2$ and the definition of $T_1$ in \eqref{prop:breg-1},  we get 
	\begin{eqnarray*}
	\epsilon & \geq & \EE[T_1+T_2]\\
	& = & \EE\bigg[\underbrace{\sum\limits_{a\in\cA}\ip{(I-\gamma P_a)v^*-r_a+c\nabla _{\lambda_a}\rho(\lambda^*),\bar \lambda_a-\lambda_a^*}}_{\geq0\mbox{ due to optimality condition }\eqref{KKT}}\bigg] + c\EE[D(\bar \lambda,\lambda^*)]\\
	&\geq& c\EE[D(\bar \lambda,\lambda^*)].
	\end{eqnarray*}
Now let us denote $\theta = (1-\gamma)$ for the ease of notation.  By direct calculation, we compute the function $D$ as follows
	\begin{eqnarray*}
		& & D(\bar\lambda,\lambda^*) \\
		& = & \rho(\bar{\lambda})-\rho(\lambda^*)-\langle\nabla\rho(\lambda^*),\bar{\lambda}-\lambda^*\rangle\\
		& = & \sum_{s\in\cS}\sum_{a\in\cA}\left(\theta\bar\lambda_{sa}\log\left(\frac{\theta\bar\lambda_{sa}}{\mu_{sa}}\right) - \theta\lambda^*_{sa}\log\left(\frac{\theta\lambda^*_{sa}}{\mu_{sa}}\right) - \left(\theta\log\left(\frac{\theta\lambda^*_{sa}}{\mu_{sa}}\right)+\theta\right)(\bar{\lambda}_{sa}-\lambda^*_{sa})\right)\\
		& = & \sum_{s\in\cS}\sum_{a\in\cA}\left(\theta\bar\lambda_{sa}\log\left(\frac{\theta\bar\lambda_{sa}}{\mu_{sa}}\right) - \theta\bar\lambda_{sa}\log\left(\frac{\theta\lambda^*_{sa}}{\mu_{sa}}\right) -\theta(\bar{\lambda}_{sa}-\lambda^*_{sa})\right)\\
		& = & KL(\theta\bar{\lambda}||\theta\lambda^*),
	\end{eqnarray*}
	where the last inequality use the fact that $\sum_{s\in\cS}\sum_{a\in\cA}\theta\bar{\lambda}_{sa} = 1 = \sum_{s\in\cS}\sum_{a\in\cA}\theta\lambda^*_{sa}.$ This completes the proof. 
\end{proof}

\section{Proof of Theorem \ref{thm:multi-stage}}\label{Appendix:theorem-non-convex}
\begin{proof}
For the ease of presentation, let us first prove two lemmas. 
\begin{lemma}
	\label{lemma:subproblems}
	For each subproblem \eqref{prob:BCD-mu}, there is a closed form solution. For \eqref{prob:BCD-lambda}, we apply Algorithm \ref{algo_one} with the number of iterations $T$ set according to Theorem \ref{thm:multi-stage}.	Then for all $\lambda^k$ and $\mu^k$ with $k\in\{1,...,K\}$, the approximate feasibility condition \eqref{defn:EpsSolu-multi-stage-2} is satisfied. Furthermore we have
	$$\EE\left[\Phi(\lambda^{k+1},\mu^{k+1}) -\max_{\lambda \in \Lambda}\Phi(\lambda,\mu^{k+1}) \right]\geq -\epsilon.$$
\end{lemma}
The proof of this lemma is provided in Appendix \ref{Appendix:lemma-subproblems}. 
\begin{lemma}
	\label{lemma:multi-stage-1}
	Suppose the sequence $\{(\lambda^k,\mu^k)\}$ is generated by Algorithm \ref{algo_BCD} with the parameters set by Theorem \ref{thm:multi-stage}. For each iteration of Algorithm \ref{algo_BCD}, let us define $$\lambda^{k+1}_* = \arg\max_{\lambda\in\Lambda} \Phi(\lambda,\mu^{k+1}).$$ Then the output $(\lambda^{k^*},\mu^{k^*})$ solution satisfies
	$$\EE[\|\lambda^{k^*}-\lambda^{k^*}_*\|^2]\leq \frac{2\epsilon}{M}\qquad\mbox{and}\qquad\EE\left[\|\lambda^{k^*}_*-\lambda^{k^*-1}\|^2\right] \leq \frac{4\epsilon}{M}.$$
\end{lemma}
The proof of this lemma is provided in Appendix \ref{appendix:multi-stage-1}. Based on this result, let us bound the expected squared projected gradients. By the optimality of $\mu^{k^*}$ for subproblem \eqref{prob:BCD-mu}, $$\Pi_U(\nabla_\mu\Phi(\lambda^{k^*-1},\mu^{k^*})) = 0.$$ Consequently,
\begin{eqnarray*}
& & \|\Pi_U(\nabla_\mu\Phi(\lambda^{k^*},\mu^{k^*}))\|^2 \\
& = & \|\Pi_U(\nabla_\mu\Phi(\lambda^{k^*},\mu^{k^*}))-\Pi_U(\nabla_\mu\Phi(\lambda^{k^*-1},\mu^{k^*}))\|^2\\
& \leq & \|\nabla_\mu\Phi(\lambda^{k^*},\mu^{k^*})-\nabla_\mu\Phi(\lambda^{k^*-1},\mu^{k^*})\|^2\\
& = & (1-\gamma)^2\|4c\langle\lambda^{k^*}-\lambda^{k^*-1},r\rangle r + 2M(\lambda^{k^*}-\lambda^{k^*-1})\|^2\\
& \leq & (1-\gamma)^2\|4crr^\top+2MI\|_2^2\|\lambda^{k^*}-\lambda^{k^*-1}\|^2\\
& = & 4(1-\gamma)^2(2c|\cS||\cA|+M)^2\|\lambda^{k^*}-\lambda^{k^*-1}\|^2\\
& \leq & 8(1-\gamma)^2(2c|\cS||\cA| + M)^2\left(\|\lambda^{k^*}_*-\lambda^{k^*-1}\|^2+ \|\lambda^{k^*}-\lambda^{k^*}_*\|^2\right).
\end{eqnarray*}
In the above arguments, the fourth row is yielded by directly substituting the formulas of $\nabla_\mu\Phi(\lambda^{k^*},\mu^{k^*})$ and $\nabla_\mu\Phi(\lambda^{k^*-1},\mu^{k^*})$; the sixth row is due to $\|rr^\top\|_2=\|r\|^2 \leq |\cS||\cA|$.  Substituting the bounds provided in Lemma \ref{lemma:multi-stage-1} yields 
\begin{eqnarray}
\label{eqn:-1}
\EE\left[\|\Pi_U(\nabla_\mu\Phi(\lambda^{k^*},\mu^{k^*}))\|^2\right] \leq \frac{48\epsilon}{M}(1-\gamma)^2(2c|\cS||\cA|+M)^2.
\end{eqnarray}
Similarly, 
\begin{eqnarray*}
	& & \|\Pi_\Lambda(\nabla_\lambda\Phi(\lambda^{k^*},\mu^{k^*}))\|^2 \\
	& = & \|\Pi_\Lambda(\nabla_\lambda\Phi(\lambda^{k^*},\mu^{k^*}))-\Pi_\Lambda(\nabla_\lambda\Phi(\lambda^{k^*}_*,\mu^{k^*}))\|^2\\
	& = & \|\nabla_\lambda\Phi(\lambda^{k^*},\mu^{k^*})-\nabla_\lambda\Phi(\lambda^{k^*}_*,\mu^{k^*})\|^2\\
	& = & (1-\gamma)^4\|2c\langle\lambda^{k^*}_*-\lambda^{k^*},r\rangle\cdot r + M(\lambda^{k^*}_*-\lambda^{k^*})\|^2\\
	&\leq & (1-\gamma)^4\|2crr^\top + MI\|_2^2\|\lambda^{k^*}_*-\lambda^{k^*}\|^2\\
	& \leq & (1-\gamma)^4(2c|\cS||\cA| + M)^2\|\lambda^{k^*}_*-\lambda^{k^*}\|^2 \\
	& \leq & \frac{16\epsilon}{M}(1-\gamma)^4(2c|\cS||\cA|+M)^2.
\end{eqnarray*}
Combine the above inequality with \eqref{eqn:-1}, we get $$\EE\big[\|\Pi_\Lambda(\nabla_\lambda \Phi(\lambda,\mu))\|^2+\|\Pi_U(\nabla_\mu \Phi(\lambda,\mu))\|^2\big] \leq \cO\left((1-\gamma)^2\left(\frac{c^2|\cS|^2|\cA|^2}{M} + M\right)\epsilon\right).$$ 
\end{proof}
\subsection{Proof of Lemma \ref{lemma:subproblems}}
\label{Appendix:lemma-subproblems}
\begin{proof}
	First, let us consider the subproblem \eqref{prob:BCD-mu} for updating $\mu$. Note that for any fixed $\lambda$, we rewrite the subproblem as follows (constants omitted) 
	$$\max_{\mu}\,\,\,2c\ip{\hat{\lambda},r}\ip{\mu,r} - c\ip{\mu,R} - \frac{M}{2}\|\mu-\hat{\lambda}\|^2\quad \text{s.t.}\,\,\,\,\mu\geq0,\|\mu\|_1 = 1.$$
	With a few more reformulation, this is actually a projection problem:
	$$\min_\mu \left\|\mu-\left(\hat{\lambda} + \frac{2}{M}(2\langle\hat \lambda, r\rangle\cdot r-R)\right) \right\|^2\quad \text{s.t.}\,\,\,\,\mu\geq0,\|\mu\|_1 = 1.$$
	This problem has a closed form solution and can be implemented within $\cO(|\cS||\cA|\log(|\cS||\cA|))$ cost, see \citep{wang2013projection}.
	
	Second, we consider the subproblem \eqref{prob:BCD-lambda}. As a special case of problem \eqref{prob:dual} whose sample complexity is fully characterized by Theorem \ref{theorem:convergence} and Theorem \ref{theorem:duality-gap-meaning}, all we need to do here is to specify the constant $\sigma$ with the following ``$\rho$ function'':
	$$\rho(\lambda) = \langle\hat\lambda,r\rangle^2 - 2\langle \mu,r\rangle\langle\hat{\lambda},r\rangle + \langle\mu,R\rangle+\frac{M}{2c}\|\hat \lambda-\mu\|^2,$$
	where $\hat \lambda = (1-\gamma)\lambda$. Note that  $\nabla \rho(\lambda) = (1-\gamma)\cdot(2(\ip{\hat\lambda,r}-\ip{\mu,r})\cdot r + \frac{M}{c}(\hat\lambda-\mu))$. Then, if we directly set $\hat{\partial}\rho(\lambda):=\nabla\rho(\lambda)$, we have
	\begin{eqnarray*}
		\sigma & = &  \sup_{\lambda\in\cL}\|\nabla\rho(\lambda)\|_\infty \leq (1-\gamma)(1 + M/c).
	\end{eqnarray*}
	Therefore, Theorem \ref{theorem:convergence}  and Theorem \ref{theorem:duality-gap-meaning} tell us that the required sample complexity is 
	\begin{eqnarray*}
		T =  \Theta\left(\frac{|\cS||\cA|\log(|\cS||\cA|)(1+c\sigma)^2}{(1-\gamma)^4\epsilon^2}\right) =  \Theta\left(\frac{|\cS||\cA|\log(|\cS||\cA|)}{(1-\gamma)^4\epsilon^2}\left(1+(1-\gamma)^2(c^2+M^2)\right)\right).
	\end{eqnarray*} 
	Thus we complete the proof. 
\end{proof}

\subsection{Proof of Lemma \ref{lemma:multi-stage-1}}
\label{appendix:multi-stage-1}
\begin{proof}
By Lemma \ref{lemma:subproblems}, we have 
\begin{equation}
\label{thm:multi-stage-1-b}
\EE[\Phi(\lambda^{k+1},\mu^{k+1}) -\Phi(\lambda_*^{k+1},\mu^{k+1}) ]\geq -\epsilon.
\end{equation}
By $M$-strongly concavity of $\Phi(\cdot,\mu^k)$, we know
$$\frac{M}{2}\|\lambda^{k+1}-\lambda_*^{k+1}\|^2 \leq \EE[\Phi(\lambda_*^{k+1},\mu^{k+1}) -\Phi(\lambda^{k+1},\mu^{k+1})]\leq\epsilon.$$
Dividing both sides by $M/2$ proves the first inequality of the lemma.  Again, by $M$-strongly concavity of $\Phi(\cdot,\mu^k)$, we also have
\begin{equation}
\label{thm:multi-stage-1-a}
\Phi(\lambda_*^{k+1},\mu^{k+1}) \geq \Phi(\lambda^{k},\mu^{k+1}) + \frac{M}{2}\|\lambda^{k+1}_*-\lambda^k\|^2.
\end{equation}
Because $\mu^{k+1} = \arg\max_{\mu} \Phi(\lambda^k,\mu)\,\,\, \mathrm{s.t.} \,\,\,\mu\geq0,\|\mu\|_1=1.$ We know 
\begin{equation}
\label{thm:multi-stage-1-c}
\Phi(\lambda^k,\mu^{k+1})\geq\Phi(\lambda^k,\mu^k).
\end{equation}
Combining the above three inequalities, we have
\begin{eqnarray*}
	&& \EE\big[\Phi(\lambda^{k+1},\mu^{k+1}) - \Phi(\lambda^{k},\mu^{k})\big]\\
	& \geq &  \EE\big[\Phi(\lambda^{k+1},\mu^{k+1}) - \Phi(\lambda^{k},\mu^{k+1})\big] \\
	& = & \EE\big[\Phi(\lambda^{k+1},\mu^{k+1}) - \Phi(\lambda^{k+1}_*,\mu^{k+1})\big]+\EE\big[\Phi(\lambda^{k+1}_*,\mu^{k+1}) - \Phi(\lambda^{k},\mu^{k+1})\big]\\
	&\geq &  \frac{M}{2}\|\lambda^{k+1}_*-\lambda^k\|^2-\epsilon,
\end{eqnarray*} 
where the second row is due to \eqref{thm:multi-stage-1-c}, the last row is due to \eqref{thm:multi-stage-1-b} and \eqref{thm:multi-stage-1-a}. Summing up the above inequalities over all $k$ and then take the average, then we know
\begin{eqnarray*}
\frac{M}{2K}\sum_{k=1}^{K} \EE\left[\|\lambda^{k}_*-\lambda^{k-1}\|^2\right] \leq\frac{\max_{\lambda,\mu}\Phi(\lambda,\mu)-\Phi(\lambda^0,\mu^0)}{K}\! +\! \epsilon\leq 2\epsilon.
\end{eqnarray*}
Because $k^*$ is randomly chosen from $\{1,...,K\}$, we get
\begin{eqnarray*}
\label{thm:multi-stage-1}
\EE\left[\|\lambda^{k^*}_*-\lambda^{k^*-1}\|^2\right]& = &	\frac{1}{K}\sum_{k=1}^{K} \EE\left[\|\lambda^{k}_*-\lambda^{k-1}\|^2\right] \leq \frac{4\epsilon}{M}.
\end{eqnarray*}
This proves the second inequality of this lemma. 
\end{proof}

\end{document}